\documentclass{article}

\usepackage{microtype}
\usepackage{graphicx}
\usepackage{subcaption}
\usepackage{booktabs} % for professional tables

\usepackage{hyperref}

\usepackage[preprint]{icml2026}

\usepackage{amsmath}
\usepackage{bm}
\usepackage{amssymb}
\usepackage{mathtools}
\usepackage{amsthm}
\usepackage{algorithmic}

\usepackage[capitalize,noabbrev]{cleveref}

\theoremstyle{plain}
\newtheorem{theorem}{Theorem}[section]
\newtheorem{proposition}[theorem]{Proposition}
\newtheorem{lemma}[theorem]{Lemma}
\newtheorem{corollary}[theorem]{Corollary}
\theoremstyle{definition}
\newtheorem{definition}[theorem]{Definition}
\newtheorem{assumption}[theorem]{Assumption}
\theoremstyle{remark}
\newtheorem{remark}[theorem]{Remark}

\newtheorem{example}{Example}

\newcommand{\E}{\mathbb{E}}

\newcommand{\Var}{\mathrm{var}}

\DeclareMathOperator*{\argmin}{arg\,min}

\newcommand{\Env}{\mathbb{E}_n}
\newcommand{\calH}{\mathcal{H}}
\newcommand{\calX}{\mathcal{X}}
\newcommand{\calY}{\mathcal{Y}}
\newcommand{\calW}{\mathcal{W}}
\newcommand{\calN}{\mathcal{N}}
\newcommand{\norm}[1]{\left\|#1\right\|_2}
\newcommand{\inner}[2]{\langle #1,\, #2 \rangle}
\icmltitlerunning{ScoreStop: Gradient-based early stopping using functional score tests}

\begin{document}

\twocolumn[
  \icmltitle{ScoreStop: Gradient-based early stopping using functional score tests}



  \icmlsetsymbol{equal}{*}

  \begin{icmlauthorlist}
    \icmlauthor{Oliver J. Hines}{columbia}
    \icmlauthor{Christian L. Hines}{turing}
  \end{icmlauthorlist}

  \icmlaffiliation{columbia}{Columbia University, NY, USA}
  \icmlaffiliation{turing}{Alan Turing Institute, London, UK}
  \icmlcorrespondingauthor{Oliver J. Hines}{oh2272@cumc.columbia.edu}

  \icmlkeywords{Boosting, Hypothesis Testing}

  \vskip 0.3in
]



\printAffiliationsAndNotice{}  % no special notice (required even if empty)

\begin{abstract}
    Gradient boosted decision trees require a stopping rule to avoid overfitting. The standard rule monitors a validation loss and stops if the loss fails to improve for a fixed patience period. However, the patience parameter has no interpretable scale and validation losses can be noisy or implicitly defined by a user-specified gradient. We propose \textbf{ScoreStop}, a gradient-based early-stopping rule that casts the stopping decision at each iteration as a test of the null hypothesis that the current predictor is the population risk minimizer. We use a functional score test, computed on validation data, with a statistic that is scale-invariant in the update direction, with a known asymptotic distribution under the null. Because our test uses gradients rather than loss values, the same construction applies to implicit losses such as LambdaRank, and data-dependent losses such as Cox regression via influence functions. In synthetic experiments and real-data benchmarks, we show that ScoreStop is competitive with loss-based methods.
\end{abstract}

\section{Introduction}
\label{sect:introduction}

Gradient boosted decision trees \citep{friedman_greedy_2001, mason_boosting_2000} are effective tabular prediction algorithms, but they are also prone to overfitting when trained for too many boosting iterations.
Tuning the number of boosting iterations is typically done through \emph{loss-based} early stopping: a (hold-out) validation loss is monitored and training is halted after a user-specified number of consecutive non-improving iterations, controlled by a \emph{patience} parameter.
However, the patience parameter is weakly calibrated, since the same patience value can behave differently across datasets, sample sizes, and loss metrics, with performance that is sensitive to other hyperparameters such as the learning rate and those controlling the complexity of the base learners.

The loss itself is also not always the right primitive.
In learning-to-rank, methods such as LambdaRank are specified through gradients, rather than losses, with the gradients approximating those of a nonsmooth target ranking metric \citep{burges_ranknet_2010}.
Moreover, losses may be used which depend on the data-generating distribution, such as in Cox regression for survival analysis, where the unit loss depends on an unknown baseline hazard.
In such settings, evaluating the validation loss requires estimating nuisance functions that may bias the stopping decision unless the loss estimators are explicitly debiased.

We propose \emph{ScoreStop}, an early-stopping criterion that tests the null hypothesis that the current function estimator is the population risk minimizer.
At each boosting iteration, the next fitted base learner supplies a test direction, and the validation sample estimates the directional derivative of the risk in that direction.
Rescaling by the estimated standard deviation gives a scale-invariant statistic that is asymptotically $\chi^2$ distributed.
The threshold can therefore be expressed on a common calibrated scale, while still being used as a regularization parameter for prediction.
To our knowledge, ScoreStop is the first boosting early-stopping rule built around a hypothesis test; the closest related proposal of \citet{Mahsereci2017} compares the empirical mean and variance of training-sample gradients using a heuristic threshold.
By using the score-test framework of \citet{hudson_inference_2026}, ScoreStop also covers implicit and data-dependent losses through the same directional-derivative principle.

We derive the ScoreStop statistic for smooth, implicit, and data-dependent losses; prove its null and alternative behavior; and evaluate it on synthetic problems with known population minimizers and real tabular benchmarks spanning regression, classification, count, ranking, quantile, and survival tasks.

\section{Preliminaries}
\label{sec:background}

\label{sec:boosting-intro}
\textbf{Gradient boosting}: Let $W = (Y, X) \in \calW$ be a random variable consisting of an outcome $Y \in \calY$ and predictors $X \in \calX$, and let $\calH$ be a Hilbert space of measurable functions $f : \calX \to \mathbb{R}$ with the $L_2$ inner product $\inner{f}{g} = \E[f(X)g(X)]$ and the implied norm $\norm{f} = \sqrt{\inner{f}{f}}$.
For a differentiable unit loss $L : \mathbb{R} \times \calW \to \mathbb{R}$, define the population risk minimizer
$f_0 = \argmin_{f \in \calH} \E[L\{f(X), W\}]$,
which is assumed to exist and be unique.
We write $\nabla L$ for the partial derivative of $L$ with respect to its first argument.
Gradient boosted estimators for $f_0$, due to \citet{mason_boosting_2000} and \citet{friedman_greedy_2001}, are initialized at function $\hat{f}_1$, typically a constant, and iterated as
$\hat{f}_{m+1}(x) = \hat{f}_m(x) - \eta_m\,\hat{h}_m(x)$,
where $\eta_m \in (0,1]$ is a step size (learning rate) and $\hat{h}_m$ is a base learner, e.g. a shallow decision tree, that approximates the gradient function
\begin{equation}
  h^*(x; \hat{f}_m) = \E\!\left[\nabla L\{\hat{f}_m(X), W\} \;\big|\; X = x\right].
  \label{eq:neg_grad}
\end{equation}
Written without arguments, we refer to this function as $h^*(\hat{f}_m)$.
In practice, the base learner is fitted by regressing $\nabla L\{\hat{f}_m(X), W\}$ on the predictors $X$ using training data.

\label{sec:score_background}
\textbf{Score tests for function-valued parameters}: For a vector-valued parameter $\theta_0$ that minimizes a population log-likelihood, the classical score test of \citet{rao_large_1948} assesses the null hypothesis $H_0: \theta_0 = \theta_*$ for a hypothesized value $\theta_*$.
This test rejects the null when the derivative of the log-likelihood with respect to $\theta_0$, evaluated at $\theta_*$, is statistically distinguishable from zero.
The score test is appealing because it does not require an estimate $\hat{\theta}$ to be obtained under the alternative $H_0: \theta_0 \neq \theta_*$.
In contrast, the classical likelihood ratio test due to \citet{wilks_large-sample_1938}, is based on the difference in the log-likelihood evaluated at $\theta_*$ and $\hat{\theta}$.

\citet{hudson_inference_2026} extend the score test to function-valued parameters $f_0 = \argmin_{f \in \calH} R_0(f)$ that minimize a population risk $R_0: \calH \to \mathbb{R}$.
For a hypothesized function $f_*$, they assess the null hypothesis $H_0: f_0 = f_*$, using a principle similar to the classical score test:
their test rejects the null when the Gateaux derivative of the risk is statistically distinguishable from zero, with the derivative evaluated in $f_*$ and over a set of directions.
For a direction $h \in \calH$, the Gateaux derivative, also called the directional derivative, is the limit $R^\prime_0(h;f) = \partial R_0(f + th) / \partial t |_{t=0}$.

For example, $R_0(f) = \E[L\{f(X), W\}]$ has the derivative $R^\prime_0(h;f) = \E[h(X) \nabla L\{f(X), W\}]$.
Since $f_0$ is a minimizer, $R^\prime_0(h;f_0) = 0$ for all possible directions $h \in \calH$, hence a test can be constructed by empirically estimating the maximum of $R^\prime_0(h;f_*)$ over a set of directions.
In practice, their test rejects the null when this maximum
exceeds a threshold that is obtained via a bootstrap-calibrated null distribution.
Like the classical score test, the functional score test does not require an estimate $\hat{f}$ to be obtained under the alternative $H_0: f_0 \neq f_*$.
The analog of the classical likelihood ratio test in the functional setting would be to base inference on an estimate of the risk difference $R_0(f_*) - R_0(\hat{f})$, which does require
the unconstrained minimizer $\hat{f}$ to be learned.

Our proposal, ScoreStop, can be viewed as an analytically calibrated instance of the functional score testing framework:
we evaluate the Gateaux derivative in the single direction provided by the boosting base learner, which is an optimal direction according to the Cauchy--Schwarz inequality.
We construct a test statistic that is chi-squared distributed under the null, without requiring the bootstrap for inference.
Our proposal is compared with loss-based early stopping, which is similar, in spirit, to the classical likelihood ratio test.

\section{The ScoreStop statistic}
\label{sec:method}

\noindent\textbf{Smooth unit losses.}
We develop an early stopping procedure that tests, using a functional score test at each boosting iteration, whether the population risk minimizer has been found.
By the Gateaux derivative formula of Section~\ref{sec:score_background}, we let
$R'_0(h;f) = \E\left[s(W; h, f)\right]$
where we define the score function $s(W; h, f) = h(X)\,\nabla L\{f(X), W\}$, and note that $R'_0(h;f_0) = 0$ for all directions $h \in \calH$.
Given validation data $\{W_i\}_{i=1}^{n}$, drawn independently of the training data, we empirically approximate $R'_0(h;f)$ as $R'_n(h;f) = \Env\left[s(W; h, f)\right]$.
Under the null $H_0: f_* = f_0$, and fixing $h\neq 0$, the Central Limit Theorem gives
\begin{align}
    \sqrt{n}R'_n(h;f_*)\overset{d}{\to} \calN \left\{0;\E\left[s(W; h, f_*)^2\right]\right\}, \label{eq:clt}
\end{align}
where $\calN\{\mu; \sigma^2\}$ denotes a normal distribution with mean $\mu$ and variance $\sigma^2$; and the variance above is assumed to be nonzero and finite.
Estimating the variance by $\Env[s(W; h, f_*)^2]$
and squaring, gives the \emph{ScoreStop statistic}
\begin{equation}
  T_n(h, f_*) =
    \frac{n\;\Env\!\left[s(W; h, f_*)\right]^2}
         {\Env\!\left[s(W; h, f_*)^2\right]}.
  \label{eq:tstat}
\end{equation}
Theoretical results for this statistic are discussed in Section~\ref{sect:theory}, but these are summarized as follows:
Proposition~\ref{prop:null} confirms that, under the null hypothesis,
$T_n(h, f_*) \overset{d}{\to} \chi^2_1$ for all directions $h$,
where $\chi^2_1$ represents a Chi-squared distribution with $1$ degree of freedom;
Proposition~\ref{prop:power} confirms that $T_n(h, f_*) \to \infty$ under the alternative hypothesis, as long as $\inner{h} {h^*(f_*)} \neq 0$, where $h^*$ is the gradient function in Section~\ref{sec:boosting-intro};
Corollary~\ref{corol:opt} motivates the direction $h = h^*(f_*)$, which is close to optimal.

Moreover, $T_n(h, f_*)$ is scale-invariant in $h$ and $L$.
That is, for any $c \neq 0$,
replacing $h$ with $ch$ or $\nabla L$ with $c \nabla L$
leaves $T_n$ unchanged.
In particular, the sign convention for $\nabla L$ does not matter, and
$T_n(\hat{h}_m, \hat{f}_m) = T_n(\hat{f}_{m+1} - \hat{f}_m, \hat{f}_m)$
regardless of the choice of learning rate $\eta_m$.

The following examples illustrate how $T_n$ specializes to several standard boosting losses; in each case the gradient and risk minimizer follow directly from the loss definition.

\begin{example}[Squared-error loss]
$L_2$-boosting \citep{buhlmann_boosting_2003} uses
$L\{f(x), w\} = \{y - f(x)\}^2$, with the gradient
$\nabla L\{f(x),w\} = -2\{y - f(x)\}$, and hence
$T_n(h, f) = n\Env\left[h(X)\{Y - f(X)\}\right]^2 / \Env\left[h(X)^2\{Y - f(X)\}^2\right]$
with risk minimizer $f_0(x) = \E[Y\mid X=x]$.
\end{example}

\begin{example}[Logistic loss]
For a binary outcome $\calY = \{-1, +1\}$,
LogitBoost \citep{friedman_logitboost_2000} uses the logistic loss
$L\{f(x),w\} = \log(1 + \exp\{-yf(x)\})$ with gradient
$\nabla L\{f(x),w\} = -y[1 + \exp\{yf(x)\}]^{-1}$, and risk minimizer
$f_0(x) = \log\frac{P(Y=1\mid X=x)}{P(Y=-1\mid X=x)}$.
\end{example}

\begin{example}[Poisson loss]
For non-negative count outcomes in $\calY \subseteq \mathbb{N}$,
the Poisson loss is $L\{f(x), w\} = \exp\{f(x)\} - y f(x)$, with the gradient
$\nabla L\{f(x), w\} = \exp\{f(x)\} - y$, and risk minimizer $f_0(x) = \log\E[Y \mid X = x]$.
\end{example}

In the examples above, the function $f:\calX \to \mathbb{R}$ has a scalar output. However, the ScoreStop framework also generalizes to vector prediction problems, where the target function has vector-valued output of dimension $K$, and the loss is a function $L: \mathbb{R}^K\times \calW \to \mathbb{R}$.
The risk minimizer $f_0$ is an element of  the Hilbert space $\mathcal{H}$ of functions $f: \calX \to \mathbb{R}^K$, with inner product $\inner{h}{f} = \E[h(X)^\top f(X)]$, where $a^\top b$ denotes the Euclidean inner product.
In this Hilbert space, the Gateaux derivative of the loss is
\begin{align}
    s(W; h, f) &= \frac{\partial L\{f(X) + th(X), W\}}{\partial t}\Big|_{t=0} \nonumber \\
    &= \sum_{k=1}^{K} h_k(X)\,\nabla_k L\{f(X), W\} \label{eq:score}
\end{align}
where $h_k$ is the $k$th component of $h$ and $\nabla_k$ is the derivative of $L$ with respect to the $k$th component of the first argument.
Hence, $R'_0(h;f) = \E[s(W; h, f)]$, which is consistent with the definition of the ScoreStop statistic in \eqref{eq:tstat}.
We illustrate this construction in the following example.

\begin{example}[Multiclass softmax loss]
\label{ex:multiclass}
For categorical outcomes $\calY = \{1, \ldots, K\}$, multiclass classification uses the loss
$L\{f(x), w\} = -f_y(x) + \log\sum_{k=1}^{K} \exp\{f_k(x)\}$. The $k$th component of the risk minimizer is
$f_{0,k}(x) = \log P(Y = k \mid X = x) + c$ for a constant $c$, which is uniquely determined by the requirement that the class probability predictions sum to 1.
The Gateaux derivative of the loss is obtained by \eqref{eq:score} with
$\nabla_k L\{f(x), w\} = \sigma_k\{f(x)\} - \mathbf{1}(y = k)$,
where $\sigma_k$ is the $k$th component of the softmax function.
\end{example}

\noindent\textbf{Nonsmooth and implicitly defined losses.}
The smooth-loss construction above assumes that $L\{f(x), w\}$ is Gateaux differentiable with respect to the prediction vector $f(x)$.
However, gradient boosting is often applied in settings where the unit loss is not differentiable, such as in quantile regression and learning-to-rank.
For quantile regression, this is achieved by using an element of the subgradient in place of the derivative.
The ScoreStop statistic readily accommodates such extensions without affecting the theoretical results in Section~\ref{sect:theory}.

\begin{example}[Quantile loss]
\label{ex:quantile}
For a user-specified quantile $\tau \in (0,1)$, the check function loss
$L\{f(x), w\} = \{y - f(x)\}[\tau - \mathbf{1}\{y < f(x)\}]$
has a risk minimizer given by the conditional quantile
$f_0(x) = F^{-1}_{Y|X}(\tau \mid x)$.
However, the check function loss is not differentiable at $y = f(x)$, but, due to the convexity of $L$, admits a subgradient $[-\tau, 1 - \tau]$ at this point.
A ScoreStop statistic can be constructed by choosing a single element of the subgradient. For example, by setting $\nabla L\{f(x), w\} = \mathbf{1}\{y \leq f(x)\} - \tau$, which implies the choice $\nabla L\{y, w\} = 1-\tau$.
\end{example}

In learning-to-rank, the goal is to learn a relevance function $f: \mathcal{X} \to \mathbb{R}$ that can be used to rank documents according to their relevance with respect to a given query.
The relevance function for a given document is called the \emph{document score}, not to be confused with the \emph{statistical scores} used in the score test.
Standard ranking metrics, such as Normalized Discounted Cumulative Gain (NDCG), are not continuous in the document scores, making learning-to-rank challenging for gradient based methods.
Specifically, the NDCG depends only on the ranking implied by the document scores, and hence is insensitive to perturbations of the document scores that do not affect the ranking.

To overcome this issue in learning-to-rank problems, such as LambdaRank by \citet{burges_lambdarank_2006}, the user specifies a gradient function that implicitly defines a loss when the Poincaré symmetry condition holds \citep[Section 4.3]{burges_ranknet_2010}. The implicit loss is often a smooth approximation of a target nonsmooth, e.g. NDCG, loss and is rarely computed itself. Instead, loss-based early stopping for learning-to-rank usually relies on the nonsmooth target loss. This may lead to situations where the implicit loss continues to improve in the validation sample, but training is halted because the nonsmooth loss does not change.
However, the ScoreStop statistic can be computed based on the user-specified gradient function and therefore may be more sensitive to improvements in the document scores that are not reflected by the nonsmooth loss.

\begin{example}[LambdaRank]
    Consider the random query $W=(N, Y, X)$ where $N \in \{1, .., n_{\max}\}$ is the number of documents.
    The vectors $X=(X_1, .. ,X_N)$ and $Y = (Y_1, ..., Y_N)$ are predictors and relevance labels for each document, with $X_k \in \calX$ and $Y_k \in \calY \subseteq \mathbb{N}$.
    For example, $\calY = \{1, 2\}$ indicates whether a document is high-relevance $(Y_k=2)$ or low-relevance $(Y_k=1)$.
    We consider a Hilbert space $\calH$ of relevance functions $f: \mathcal{X} \to \mathbb{R}$, and let $\underline{f}_N(X) = (f(X_1),\ldots, f(X_N))$ denote the document scores for a given query.
    The LambdaRank gradient for the $k$th document is a user-specified version $\lambda_k\{\underline{f}_N(X), W\}$, which plays the role of a gradient $\nabla_k L$ in analogy with \eqref{eq:score}. By the Poincaré Lemma, $\lambda_k$ is the gradient of an implied loss $L$ when, for all $j,k$
    \begin{align*}
        \frac{\partial\lambda_k\{\underline{f}_N(X), W\}}{\partial f(X_j)} = \frac{\partial\lambda_j\{\underline{f}_N(X), W\}}{\partial f(X_k)}.
    \end{align*}
    Perturbing the implied loss $L$ in the direction $h \in \calH$, we define
    \begin{align*}
        s(W; h, f) &= \sum_{k=1}^{N} h(X_k)\lambda_k\{\underline{f}_N(X), W\},
    \end{align*}
    which equals the Gateaux derivative
    \begin{align*}
        \frac{\partial L\{\underline{f}_N(X) + t \underline{h}_N(X), W\}}{\partial t}\Big|_{t=0}
    \end{align*}
    when the Poincaré symmetry condition holds.
    This expression for $s(W; h, f)$ can be used to estimate the ScoreStop statistic in \eqref{eq:tstat}. See Appendix~\ref{app:lambdarank} for detailed expressions for $\lambda_j$ and the nonsmooth NDCG-based target loss.
\end{example}

\noindent\textbf{Data-dependent losses.}
The ScoreStop constructions above assume the score function $s: \calW\times \calH^2 \to \mathbb{R}$ is known.
Here, we consider more complicated losses, where the score function depends on the data-generating distribution $P_0$, e.g. in Cox regression for survival analysis the gradient depends on the baseline cumulative hazard, which can be treated as a nuisance function.
Following the general framework of \citet{hudson_inference_2026}, we show that ScoreStop extends naturally to such settings by replacing
the score contribution $s(W; h,f)$ with an
\emph{influence function} that accounts for data dependence. See e.g. \citet{hines_demystifying_2022} for an introduction to influence functions.

First, consider that the validity of the ScoreStop statistic is based on the CLT result in \eqref{eq:clt}. This result depends on the estimator $\hat{R}_{n}^\prime(h; f_*)$ of $R^\prime_{0}(h; f_*)$ being \emph{regular asymptotically linear} under the null $H_0: f_* = f_0$, which means that
\begin{align}
    \hat{R}_{n}^\prime(h; f_*) = \Env[\varphi(W; h, f_*, P_0)] + \epsilon_n \label{eq:ral}
\end{align}
for an influence function $\varphi(W; h, f_*, P_0)$ that is mean-zero under the null, and has finite variance. Note that $\epsilon_n$ denotes a term such that $\sqrt{n}\epsilon_n = o_p(1)$, where $o_p(1)$ means: converges to zero in probability as $n \to \infty$.
The \emph{regular} qualifier here is the standard semiparametric notion \citep{tsiatis_semiparametric_2006}, which strengthens \eqref{eq:ral} with an invariance condition under local perturbations of $P_0$ that is automatic for the estimators we consider.
When the score function is known then the estimator $\Env[s(W; h, f_*)]$ is regular asymptotically linear with influence function
$\varphi(W; h, f_*, P_0) = s(W; h, f_*)$ and $\epsilon_n = 0$.
We illustrate how the influence function departs from this form in survival regression. See e.g. \citet{ridgeway_state_1999} and \citet[Section 5.2]{tsiatis_semiparametric_2006} for references.

\begin{example}[Cox proportional hazards model]
\label{ex:cox}
    Survival analysis considers iid units, with baseline predictors $\calX$, which are observed until they either have a failure event or a censoring event, e.g. because no failure event has occurred by the end of the experimental observation period.
    This data is encoded as $W = (Y, X)$, where $Y = (D,T)$ consists of an event indicator $D \in \{0, 1\}$ and an observed time $T\in\mathbb{R}_+$, i.e. $D = 1$ implies failure at time $T$ and $D = 0$ implies censoring at time $T$.
    % It is assumed that the censoring time is conditionally independent of the event time given $X$.
    Under the Cox proportional hazards model the cumulative event hazard is modeled as $\exp\{f_0(x)\}\Lambda_0(t)$, for an unknown baseline hazard function $\Lambda_0:\mathbb{R}_+\to\mathbb{R}_+$, and log-hazard ratio $f_0:\calX \to \mathbb{R}$ in a Hilbert space $\calH$. The Gateaux derivative of the unit loss is
    \begin{align*}
        s(W; h, f, \Lambda_0) = h(X)[\exp\{f(X)\}\Lambda_0(T) - D]
    \end{align*}
    which depends on $\Lambda_0$.
    Estimating $R^\prime_{0}(h; f_*) = \E[s(W; h, f_*, \Lambda_0)]$ requires an estimator for $\Lambda_0$, typically the Breslow estimator $\hat{\Lambda}_{n,f_*}$.
    Reasoning under the null $H_0: f_* = f_0$,
    the estimator $\hat{R}_n^\prime(h; f_*) = \Env[s(W; h, f_*, \hat{\Lambda}_{n,f_*})]$
    is regular asymptotically linear with influence function
    \begin{align*}
        \varphi(W; h, f_*, P_0) &= s(W; h, f_*, \Lambda_0)
        + D \bar{h}(T) \\
        &\quad - \exp\{f_*(X)\}\int_{0}^T\bar{h}(t)\mathrm{d}\Lambda_{0}(t),
    \end{align*}
    where $\bar{h}(t)$ is a weighted expectation of $h(X)$ that depends on $f_*$ and is defined in Appendix~\ref{app:survival} along with detailed derivations. Heuristically, the additional terms after $s(W; h, f_*, \Lambda_0)$ account for the uncertainty in estimating $\Lambda_0$ and could be discarded if $\Lambda_0$ were known a priori.
\end{example}

To extend the ScoreStop statistic to data-dependent losses we restrict attention to the \emph{one-step debiased estimator} of $R_0^\prime(h; f_*)$, which for a distribution estimator $\hat{P}_n$ is $\hat{R}'_n(h;f_*) = \Env[\varphi(W; h, f_*, \hat{P}_n)]$.
In practice it is not necessary that $\hat{P}_n$ estimates the full distribution $P_0$, only the components, such as $\Lambda_0$, that appear in the influence function.
For the one-step debiased estimator, formal statistical results are used to bound the error $\epsilon_n = \Env[\varphi(W; h, f_*, \hat{P}_n) - \varphi(W; h, f_*, P_0)]$ in terms of the error rates in estimating components, such as $\Lambda_0$, and the complexity of the estimator $\hat{P}_n$.
We will defer such details and assume that $\sqrt{n} \epsilon_n$ converges to zero, and hence the one-step debiased estimator is regular asymptotically linear as in \eqref{eq:ral}.
Conveniently, for the Breslow estimator in Example \ref{ex:cox}, the estimator $\Env[s(W; h, f_*, \hat{\Lambda}_{n,f_*})]$ is exactly the one-step debiased estimator.

To derive a general ScoreStop statistic we use a CLT result similar to \eqref{eq:clt} under the null $H_0: f_* = f_0$
\begin{align}
    \sqrt{n}R'_n(h;f_*)\overset{d}{\to} \calN \left\{0;\E\left[\varphi(W; h, f_*, P_0)^2\right]\right\},\label{eq:clt_if}
\end{align}
where the variance is assumed to be nonzero and finite when $h \neq 0$.
Estimating this variance as $\Env[\varphi(W; h, f_*, \hat{P}_n)^2]$, we define the general form of the ScoreStop statistic
\begin{align}
    T_n(h, f_*) = \frac{n\Env[\varphi(W; h, f_*, \hat{P}_n)]^2}{\Env[\varphi(W; h, f_*, \hat{P}_n)^2]}, \label{eq:tstat_if}
\end{align}
which reduces to \eqref{eq:tstat} when the score function is known.
In this way, we are able to naturally apply score-based early stopping to data-dependent losses.

\begin{remark}
    When the loss is data-dependent, applying loss-based early stopping is similarly challenging, since estimated nuisance functions are required to evaluate the validation risk $R_0(f_*)$.
    For instance, in Example \ref{ex:cox}, the validation risk and the ScoreStop statistic both rely on an estimator for the baseline hazard $\Lambda_0$.
    In both cases, the Breslow estimator $\hat{\Lambda}_{n,f_*}$ can be used, which is based on the validation sample.
    In particular, this means that the baseline hazard estimator used for training and validation are different.
    Generally, plugging-in an estimated nuisance function to estimate the risk may introduce bias if the uncertainty in the nuisance function estimator is not properly acknowledged, e.g. by using a one-step debiased estimator.
\end{remark}

\section{Theoretical Results}
\label{sect:theory}

In this Section we examine asymptotic results for the ScoreStop statistic in the limit as the sample size $n\to\infty$.
We center our analysis around the influence curve representation in \eqref{eq:tstat_if}, thereby allowing our results to extend to implicit and data-dependent losses as discussed in Section~\ref{sec:method}.
For a distribution estimator $\hat{P}_n$, we use the shorthand $\hat{\varphi}(h,f_*) = \varphi(W; h, f_*, \hat{P}_n)$ and $\varphi_0(h,f_*) = \varphi(W; h, f_*, P_0)$.
We require the following regularity assumptions, which encode the relevant properties of the Gateaux derivative $R^\prime_0(h;f_*)$ in terms of the influence function $\varphi$.

\begin{assumption}[Influence function centrality]
\label{assump:central}
    For all $h \in \calH$, the identity $\E[\varphi_0(h,f_0)] = 0$ holds.
\end{assumption}

\begin{assumption}[Direction linearity]
\label{assump:direction}
    The map $h \mapsto \E[\varphi_0(h,f_*)]$ is a continuous linear functional, and hence, by the Riesz representation theorem, can be represented as $\inner{h}{h^*(f_*)}$ for a unique $h^*(f_*) \in \calH$. Note that $h^*(f_*) \in \calH$ requires that $\norm{h^*(f_*)} < \infty$.
\end{assumption}

\begin{assumption}[Finite variance]
\label{assump:variance}
    For all $h \in \calH$, with $h\neq 0$, $\E[\varphi_0(h,f_0)^2] \in (0, \infty)$.
\end{assumption}

Interpreting these assumptions for the smooth unit loss $\varphi_0(h, f) = h(X)\nabla L\{f(X), W\}$, Assumption~\ref{assump:central} is satisfied since $f_0$ is the risk minimizer, and the Riesz representing function in Assumption~\ref{assump:direction} recovers the gradient function $h^*(x;f)$ in \eqref{eq:neg_grad} that is learned by the base learner in the standard gradient boosting algorithm.
This intuition for the smooth unit loss extends to all of the examples in the current paper.
Assumption~\ref{assump:direction} requires $h^*(x;f_*)$ to have finite variance, which, along with Assumption~\ref{assump:variance}, precludes pathological data situations, e.g. for the squared-error loss when $\E[\Var(Y|X)]$ is unbounded.
Finally, we make an additional assumption, that relates to the regular asymptotic linearity expression in \eqref{eq:ral}.

\begin{assumption}[Regular asymptotic linearity]
\label{assump:ral}
    For fixed $h$; $\sqrt{n}\epsilon_n = o_p(1)$, where $\epsilon_n = \Env[\hat{\varphi}(h,f_0) - \varphi_0(h,f_0)]$.
\end{assumption}

\begin{assumption}[Moment consistency]
\label{assump:consistency}
    Fix $h$ and $f_*$; $\Env[\hat{\varphi}(h, f_*)^k - \varphi_0(h, f_*)^k] = o_p(1)$ for $k = 1, 2$.
\end{assumption}

Assumptions~\ref{assump:ral} and \ref{assump:consistency} are relevant for data-dependent losses, where components of $\varphi_0$ must be estimated.
The term $\sqrt{n}\epsilon_n$ is typically analyzed on a case-by-case basis by decomposing as an empirical process $\sqrt{n}(\Env - \E)\{\hat{\varphi}(h,f_0) - \varphi_0(h,f_0)\}$ and a remainder $\sqrt{n}\E[\hat{\varphi}(h,f_0) - \varphi_0(h,f_0)]$. Controlling the empirical process usually requires restricting the complexity of the estimator $\hat{\varphi}$, or using cross-fitting, while the remainder usually requires conditions on the rates of convergence of nuisance functionals. See for example \citet{hines_demystifying_2022} and \citet{laan_researchers_2026}.

Under these assumptions we derive the following statements which characterize the behavior of the ScoreStop statistic $T_n(h, f_*) = n\Env[\hat{\varphi}(h, f_*)]^2/\Env[\hat{\varphi}(h, f_*)^2]$ under the null and the alternative hypothesis. See Appendix~\ref{app:proofs} for proofs.

\begin{proposition}[Distribution under the null]
\label{prop:null}
Under the null hypothesis $H_0: f_* = f_0$,
$T_n(h, f_*) \overset{d}{\to} \chi^2_1$ as $n \to \infty$ for fixed $h$, where $\chi^2_d$ denotes the chi-squared distribution with $d$ degrees of freedom.
\end{proposition}

\begin{proposition}[Power consistency]
\label{prop:power}
Fix $h \neq 0$ and assume $\E[\varphi_0(h,f_*)^2] \in (0, \infty)$.
Under the alternative hypothesis $H_0: f_* \neq f_0$,
$P\{T_n(h, f_*) > c_\alpha\} \to 1$ as $n\to \infty$ for any fixed $c_\alpha$
provided that $\inner{h}{h^*(f_*)} \neq 0$.
\end{proposition}

Together, these propositions confirm the validity for hypothesis testing based on the ScoreStop statistic $T_n(h, f_*)$. However, Proposition~\ref{prop:power} highlights the sensitivity of this test to the direction $h$.
In particular, for a fixed direction $h$, tests based on $T_n(h, f_*)$ have no power to test against alternatives $f_*$ that satisfy $\inner{h}{h^*(f_*)} = 0$.
Similar problems for conditional independence testing are reported by \citet{shah_hardness_2020}, who show that it is impossible to construct a test that has power over \emph{any} alternative.
The direction selection problem can be partly resolved by choosing a direction $h$ that maximizes $|\inner{h}{h^*(f_*)}|$.

\begin{corollary}[Approximately optimal direction]
\label{corol:opt}
    By the Cauchy--Schwarz inequality, $|\inner{h}{h^*(f_*)}|$ is maximized when $h$ is proportional to $h^*(f_*)$.
    Thus, for such a direction $h$, $h^*(f_*) \neq 0$ and $\E[\varphi_0(h,f_*)^2] \in (0, \infty)$ guarantees power consistency in Proposition~\ref{prop:power}.
\end{corollary}

Through Corollary~\ref{corol:opt}, we see that the direction $h^*(f_*)$, targeted by the gradient boosting base learner, provides improved guarantees of power consistency. We say that this direction is \emph{approximately} optimal since it maximizes the ScoreStop statistic numerator $\Env[\hat{\varphi}(h, f_*)]^2$ in the population limit. However, to derive an optimal direction for the statistic itself one must also consider the denominator $\Env[\hat{\varphi}(h, f_*)^2]$, which we assume is finite and nonzero.
\citet[Section 4.2]{hudson_inference_2026} discusses direction optimality with regards to local asymptotic power, where $\E[\varphi_0(h,f_*)] = \sqrt{n}c$, for a constant $c$, and a local regularity condition holds, similar to \eqref{eq:ral}.
For the purposes of early stopping, we consider the gradient boosting learner from the training sample to be a practical and well motivated choice.

Finally, consider that an alternative approach to improving power consistency is to test in multiple directions.
For example, let $\varphi_0(\bm{h}, f)$ denote the vector $(\varphi_0(h_1, f), ..., \varphi_0(h_d, f))$ for a direction set $\{h_1, ..., h_d\}$ of size $d$. The ScoreStop statistic naturally extends to multiple directions as $T_n(\bm{h}, f_*) = n\Env[\hat{\varphi}(\bm{h}, f_*)]^\top \Env[\hat{\varphi}(\bm{h}, f_*)\hat{\varphi}(\bm{h}, f_*)^\top]^{-1}\Env[\hat{\varphi}(\bm{h}, f_*)]$.
In Appendix~\ref{app:proofs} we provide results, similar to those above, for this multi-direction statistic.
In particular, under the null hypothesis, $T_n(\bm{h}, f_*) \to \chi^2_d$, where we assume that the matrix $\E[\varphi_0(\bm{h}, f_*)\varphi_0(\bm{h}, f_*)^\top]$ is invertible, i.e. a multi-direction version of Assumption~\ref{assump:variance}.

\section{Algorithm and Calibration}
\label{sect:algorithm}

We propose Forward ScoreStop (FWD-SS) as described in Algorithm \ref{alg:fwd}.
This algorithm performs gradient boosting as usual, but, at each iteration, formally tests whether to accept the update $\hat{h}_m$, using a score test with null $\hat{f}_m = f_0$ and an external validation set.
This algorithm is \emph{forward} in the sense that the ScoreStop statistic $T_n(\hat{f}_{m+1} - \hat{f}_m, \hat{f}_m)$ is evaluated in the update direction, which is approximately optimal by Corollary \ref{corol:opt}.
A related Backward ScoreStop (BWD-SS) algorithm instead tests in the direction of the previous tree, i.e. $T_n(\hat{f}_{m} - \hat{f}_{m-1}, \hat{f}_m)$.
We also propose a multi-direction alternative called Stabilized ScoreStop (STAB-SS), that tests in both the forward and backward directions.
We call this algorithm \emph{stabilized}, since the reliance on any single base learner is reduced.
Pseudo-code for BWD-SS and STAB-SS are provided in Appendix~\ref{app:pseudo-code}. These are similar to Algorithm \ref{alg:fwd} but differ in the computation of $T$, and do not test on the $m=1$ iteration, since the backward direction is not yet defined.

\begin{algorithm}[tb]
  \caption{Forward ScoreStop (FWD-SS)}
  \label{alg:fwd}
  \begin{algorithmic}
    \REQUIRE Training data $\mathcal{D}_\mathrm{tr}$,
             validation data $\mathcal{D}_\mathrm{val}$,
             gradient $\nabla L$,
             threshold $c_\alpha > 0$,
             maximum iterations $M$, learning rate sequence $\{\eta_m\}_{m=1}^{M}$, initial estimator $\hat{f}_1$.
    \FOR{$m = 1, \ldots, M$}
      \STATE \textbf{Fit base learner on training data:}\\
             \quad $\hat{h}_m(x) \leftarrow
               \mathrm{Tree}\!\left(
                 \nabla L\{\hat{f}_m(X_i), W_i\},\; i \in \mathcal{D}_\mathrm{tr}
               \right)$
      \STATE \textbf{Evaluate score statistic on validation data:}\\
              \quad $T \leftarrow T_n(\hat{h}_m, \hat{f}_m)$
      \IF{$T \leq c_\alpha$}
        \RETURN $\hat{f}_m$
      \ENDIF
    \STATE $\hat{f}_{m+1}(x) \leftarrow \hat{f}_m(x) - \eta_m\,\hat{h}_m(x)$
    \ENDFOR
    \RETURN $\hat{f}_M$
  \end{algorithmic}
\end{algorithm}

\begin{remark}[Small learning rate collinearity]
When the learning rate is small, successive boosting updates change the model slightly, and the directions $\hat{h}_m$ and $\hat{h}_{m+1}$ are almost the same. In this regime: FWD-SS and BWD-SS are expected to perform similarly, since they use almost the same direction; and the covariance matrix in STAB-SS approaches rank 1, making it singular. Based on previous studies of Wald-type statistics with rank-deficient covariance, we expect the asymptotic distribution of the STAB-SS ScoreStop statistic to deviate from $\chi^2_2$ in the near-collinearity limit \citep{andrews_asymptotic_1987, dufour_wald_2025}. See Appendix~\ref{app:experiments} for numerical investigations into this issue.
\end{remark}

\textbf{Choosing the threshold.}
Applying Proposition~\ref{prop:null}, the threshold $c_\alpha$ can be related to a significance level via $c_\alpha = F^{-1}_{\chi^2_d}(1 - \alpha)$, where $d$ is the number of directions; $d=1$ for FWD-SS / BWD-SS, and $d=2$ for STAB-SS.
% For example, $\alpha = 0.05$ and $d=1$ gives $c_\alpha \approx 1.96^2$, with 1.96 standard deviations often appearing in the construction of 95\% confidence intervals in applied statistics.
Usually in hypothesis testing scenarios, however, one chooses $\alpha$ to be small, so that the null is only rejected when there is strong statistical evidence.
In applied statistics, confidence intervals/sets are obtained by inverting this logic, returning the set of values for which have high ($1 - \alpha$) coverage of the null, e.g.\ $\alpha = 0.05$ is a common choice which corresponds to the 95\% confidence interval.
Our goal, however, is not confidence set estimation but \emph{point estimation}: we seek a single function $\hat{f}$ that is as close to $f_0$ as possible.
The point estimate can be viewed as a function that is contained in even the smallest possible confidence sets, i.e.\ the point estimate is not rejected even when $\alpha$ is large, for instance when $\alpha = 0.95$, the resulting 5\% confidence set should contain the point estimate.
In practice, $c_\alpha$ should therefore be regarded as a regularization parameter that is calibrated by the $\chi^2_d$ distribution but chosen to minimize prediction error rather than to control a type-I error rate.
In particular, the choice of $c_\alpha$ controls the early stopping trade-off:
large $c_\alpha$ corresponds to a wide confidence set (small $\alpha$) that may stop training too early;
small $c_\alpha$ corresponds to a tight confidence set (large $\alpha$) that permits more boosting iterations but may lead to overfitting.
To aid interpretability, we communicate the significance level via the two-sided $z$-score $z=\Phi^{-1}(1-\alpha/2)$, i.e. the number of standard deviations of a normal random variable. E.g. for $d=1$, $c_\alpha=z^2$.
The experimental thresholds $z \in \{0.025, 0.05, 0.1, 0.2, 0.3\}$ are used in Section~\ref{sect:experiments},
which we interpret as, e.g.: ``stop training when $\hat{f}_m$ is within 0.05 standard deviations of an ideal, point estimator''.
Appendix Table~\ref{tab:threshold-mapping} gives the numerical mapping from $z$ to $c_\alpha$ for $d \in \{1, 2\}$.

\section{Numerical experiments}
\label{sect:experiments}

\textbf{Synthetic data experiments.} We study four synthetic problems with known population minimizers $f_0$: continuous outcome regression using the squared error loss, binary classification, learning-to-rank, and survival analysis using proportional hazards models; see Appendix~\ref{app:dgp} for details.
For each task, we draw random synthetic datasets which are split into train ($n_\mathrm{tr}=2{,}000$), validation ($n_\mathrm{val}=500$), and test ($n_\mathrm{te} = 10{,}000$) sets.
We compare ScoreStop (FWD/BWD/STAB) and loss-based stopping rules for gradient boosted trees with a maximum of $M=2{,}000$ boosting iterations and constant learning rate $\eta=0.05$.
Our findings are robust to the learning rate, as we demonstrate through a learning rate sensitivity study in Appendix~\ref{app:lr-sensitivity}.
Figure~\ref{fig:illustrative-regression} shows illustrative trajectories for the FWD-SS statistic and the validation/test loss for a single regression seed.
We see that the FWD-SS statistic is large during the early iterations and decreases as $\hat f_m$ approaches the population minimizer.
Once boosting has converged the FWD-SS statistic fluctuates around the support of the $\chi^2_1$ distribution, see Appendix~\ref{app:null-calibration} for analysis of the null distribution.

\begin{figure}[tb]
  \centering
  \includegraphics[width=\columnwidth]{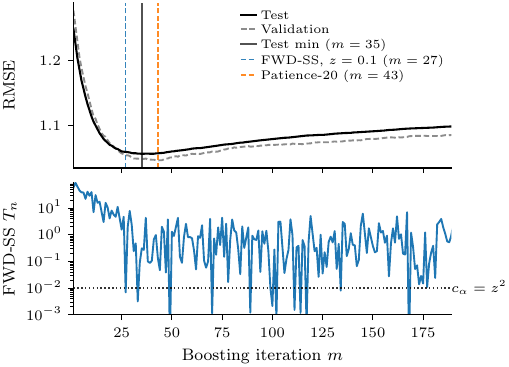}
  \caption{Single regression trajectory at $\eta = 0.05$. Top: validation and test RMSE, with vertical lines for the test-loss minimum, FWD-SS at $z=0.1$, and Patience-20. Bottom: FWD-SS statistic $T_n$ with threshold $c_\alpha=z^2=0.01$.}
  \label{fig:illustrative-regression}
\end{figure}

For each stopping method, we compare the test loss of the returned function to the test loss of the known population minimizer $f_0$, with the difference referred to as the \emph{excess test loss}.
Figure~\ref{fig:synthetic-summary} summarizes median excess test loss for each stopping method and task over $100$ Monte Carlo seeds; numeric values and median return iterations are in Appendix Table~\ref{tab:synthetic-combined-main}.
These results show that, across all tasks, ScoreStop is competitive with loss-based methods when the patience value is small, but is mostly beaten by loss-based early stopping when the patience value is large.
However, ScoreStop is particularly strong on learning-to-rank, where loss-based stopping rules perform poorly because they target a nonsmooth loss that is only indirectly tied to the training gradient.
We see that STAB-SS usually outperforms FWD-SS and BWD-SS, especially on regression and survival tasks, perhaps due to its reduced reliance on any single base learner.
Based on these experiments, we recommend STAB-SS and FWD-SS with thresholds between $z=0.05$ and $z=0.1$ as strong ScoreStop variants for general use.

\begin{figure*}[tb]
  \centering
  \includegraphics[width=\textwidth]{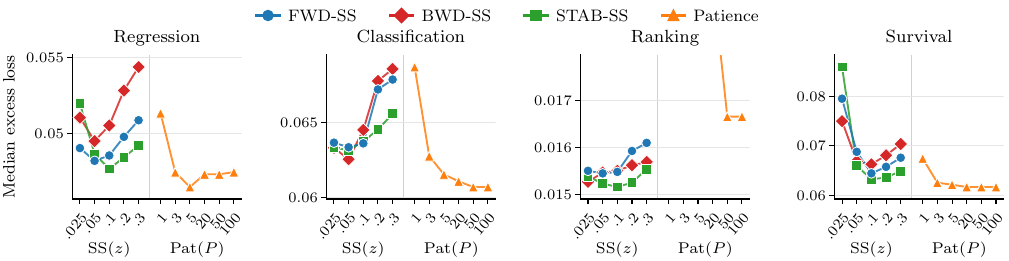}
  \caption{Synthetic experiments at $\eta=0.05$: median excess test loss over the known population minimizer $f_0$ for ScoreStop thresholds and patience values. Full numeric results and returned iterations are in Appendix Table~\ref{tab:synthetic-combined-main}.}
  \label{fig:synthetic-summary}
\end{figure*}

\textbf{Real data benchmarks.} We evaluate on thirteen benchmark tabular dataset problems spanning regression, binary and multiclass classification, Poisson regression for count data, quantile regression, ranking, and survival.
All datasets except MEPS and MSLR-WEB10K are from OpenML \citep{vanschoren2013openml}; MEPS is from \citet{ahrq2023mepshc233}; MSLR-WEB10K is from \citet{qin_introducing_2010}.
For non-ranking datasets we use 10-fold cross-validation, splitting the non-test folds into training and validation sets in a 75/25 ratio; classification, multiclass, and survival folds are stratified.
MSLR-WEB10K uses its five predefined folds.
All models use LightGBM with 31 leaves and learning rate $\eta=0.05$.
Training is extended for $1{,}000$ rounds after the hold-out test-loss minimum has been observed, with a maximum of $M=20{,}000$ rounds; the stopping rules under evaluation never see the test loss.
Because $f_0$ is unknown, Figure~\ref{fig:real-data-summary} and Appendix Table~\ref{tab:real_data} report median excess test loss over an ``oracle loss'', which is the minimum test loss achieved on the test fold; min--max fold ranges are also shown in the figure.

\begin{figure*}[tb]
  \centering
  \includegraphics[width=\textwidth]{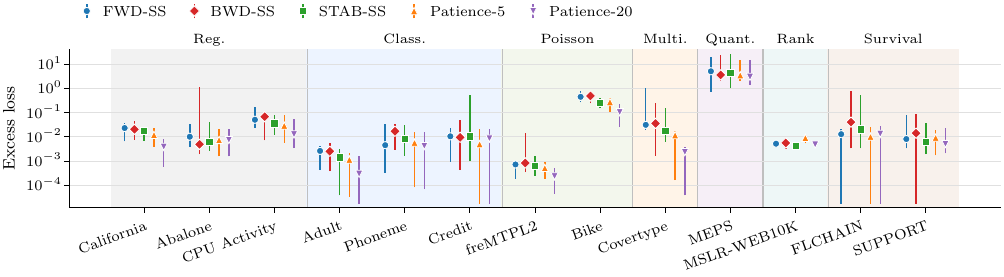}
  \caption{Real-data benchmarks: excess test loss over the retrospective test-loss oracle. Markers are fold medians and vertical bars span fold minima and maxima. ScoreStop uses fixed threshold $z=0.05$; patience baselines use $P \in \{5,20\}$.}
  \label{fig:real-data-summary}
\end{figure*}

The real-data results are broadly consistent with the synthetic experiments:
ScoreStop is competitive with loss-based stopping rules, when the patience value is small.
The three ScoreStop variants perform similarly, with no clear winner across all datasets.

\section{Discussion and related work}
\label{sect:discussion}

\textbf{Online versus lookback stopping.}
The ScoreStop algorithms here are \emph{online}: they commit to $\hat f_m$ at iteration $m$ and do not revisit earlier iterates.
Patience is a \emph{lookback} rule: after $P$ non-improving iterations it returns the validation-loss minimizer among the recent trajectory.
This distinction partly explains why patience can be strong when validation loss is smooth and reliable, while ScoreStop can be useful when the validation loss is noisy or only indirectly tied to the training gradient.
A lookback ScoreStop variant is therefore a natural extension.

\textbf{Anytime-valid inference.}
ScoreStop uses a sequence of validation-set score statistics, but we do not claim anytime-valid type-I error control.
Confidence sequences, e-values, and related methods preserve validity across data-dependent stopping times by constructing martingales or nonnegative supermartingales adapted to a sequential filtration \citep{howard_time-uniform_2021, ramdas_game-theoretic_2023}.
Recent work has used this perspective to derive stopping rules for learning algorithms. For example, \citet{sommer2026evaluebde} formulate the construction of Bayesian deep ensembles as a sequential anytime-valid hypothesis test based on e-values, while \citet{aolaritei2025sgdcs} construct anytime-valid confidence sequences for stochastic gradient methods.
These works target optional-stopping-valid inference over a sequential process. In contrast, ScoreStop uses a fixed-time functional score statistic repeatedly as a score-calibrated regularization criterion for gradient boosting.
The main challenge in developing anytime-valid inference is that the sequence of ScoreStop statistics, e.g. $\{T_n(\hat h_m, \hat f_m)\}_{m=1}^M$, is computed on a single validation sample, meaning that successive statistics are not independent.

\textbf{Early stopping as regularization.}
Prior work treats early stopping as implicit regularization with prediction-error guarantees, including $L_2$Boost \citep{buhlmann_boosting_2003}, boosted classifiers \citep{zhang_boosting_2005}, and kernel or gradient methods \citep{raskutti_early_2014, wei_early_2019}.
Cross-validation is another stopping criterion but requires repeated model fits \citep{buhlmann_hothorn_2007}.
Connecting ScoreStop thresholds to these prediction-error analyses remains open.

\textbf{Confidence sets by test inversion.}
\citet{hudson_inference_2026} originally motivate the functional score test by inverting it to obtain a confidence set for $f_0$; i.e., a set of functions $f_*$ that are not rejected by the score test at significance level $\alpha$.
Their test takes the supremum over a user-specified set of direction functions; Corollary~\ref{corol:opt} suggests a single gradient-boosted base learner trained on external data may serve as an approximately optimal direction.

\textbf{Other iterative learning algorithms.}
The score test framework is not specific to gradient boosting: the function difference $\hat h_m = \hat f_{m+1} - \hat f_m$ defines a natural test direction for any iterative procedure, including neural networks fitted with gradient descent.
\citet{Mahsereci2017} propose a related criterion that compares the parameter-gradient magnitude to its empirical variance but lacks a calibrated null; algorithmic-stability bounds \citep{hardt_train_2016} and the bias-variance analysis of \citet{yao_early_2007} for kernel methods provide further alternatives.

\section{Conclusion}
\label{sect:conclusion}

We introduced ScoreStop, a gradient-based early-stopping criterion that casts each stopping decision as a functional score test.
Its statistic is calibrated by an asymptotic $\chi^2_d$ null distribution, is computed from the gradient rather than the loss, and extends to implicit and data-dependent losses through influence functions.
Across synthetic and real-data benchmarks, a fixed ScoreStop threshold provides a competitive alternative to patience-based stopping without tuning a dataset-specific patience value.
Future work includes anytime-valid variants, lookback variants, sharper links to prediction-error regularization theory, and applications beyond gradient boosting.

% \section*{Software and Data}

% If a paper is accepted, we strongly encourage the publication of software and
% data with the camera-ready version of the paper whenever appropriate. This can
% be done by including a URL in the camera-ready copy. However, \textbf{do not}
% include URLs that reveal your institution or identity in your submission for
% review. Instead, provide an anonymous URL or upload the material as
% ``Supplementary Material'' into the OpenReview reviewing system. Note that
% reviewers are not required to look at this material when writing their review.

% \section*{Acknowledgements}

% \textbf{Do not} include acknowledgements in the initial version of the paper

% \section*{Impact Statement}

% This paper presents work whose goal is to advance the field of Machine Learning. There are many potential societal consequences of our work, none which we feel must be specifically highlighted here.


\bibliography{refs.bib}
\bibliographystyle{icml2026}

%%%%%%%%%%%%%%%%%%%%%%%%%%%%%%%%%%%%%%%%%%%%%%%%%%%%%%%%%%%%%%%%%%%%%%%%%%%%%%%
%%%%%%%%%%%%%%%%%%%%%%%%%%%%%%%%%%%%%%%%%%%%%%%%%%%%%%%%%%%%%%%%%%%%%%%%%%%%%%%
% APPENDIX
%%%%%%%%%%%%%%%%%%%%%%%%%%%%%%%%%%%%%%%%%%%%%%%%%%%%%%%%%%%%%%%%%%%%%%%%%%%%%%%
%%%%%%%%%%%%%%%%%%%%%%%%%%%%%%%%%%%%%%%%%%%%%%%%%%%%%%%%%%%%%%%%%%%%%%%%%%%%%%%
\newpage
\appendix
\onecolumn
\section{Numerical experiments}
\label{app:experiments}

\subsection{Algorithm pseudo-code}
\label{app:pseudo-code}

Algorithm~\ref{alg:bwd-stab} shows the backward (BWD-SS) and stabilized (STAB-SS) ScoreStop variants.
They differ from Algorithm~\ref{alg:fwd} only in the tested directions.
BWD-SS uses the previous update direction, while STAB-SS tests the current and previous directions jointly.
Both defer the first test to $m \geq 2$, since no previous tree is available at $m=1$.

\begin{algorithm}[H]
  \caption{Backward (left) and Stabilized (right) ScoreStop.
    Inputs (omitted) are identical to Algorithm~\ref{alg:fwd}.}
  \label{alg:bwd-stab}
  \begin{minipage}[t]{0.48\linewidth}
    \centering\textbf{BWD-SS}
    \begin{algorithmic}
      \FOR{$m = 1, \ldots, M$}
        \STATE \textbf{Fit base learner on training data:}\\
               \quad $\hat{h}_m(x) \leftarrow
                 \mathrm{Tree}\!\left(
                   \nabla L\{\hat{f}_m(X_i), W_i\},\; i \in \mathcal{D}_\mathrm{tr}
                 \right)$
        \IF{$m \geq 2$}
          \STATE \textbf{Evaluate score statistic on validation data:}\\
                  \quad $T \leftarrow T_n(\hat h_{m-1},\, \hat{f}_m)$
          \IF{$T \leq c_\alpha$}
            \RETURN $\hat{f}_m$
          \ENDIF
        \ENDIF
      \STATE $\hat{f}_{m+1}(x) \leftarrow \hat{f}_m(x) - \eta_m\,\hat{h}_m(x)$
      \ENDFOR
      \RETURN $\hat{f}_M$
    \end{algorithmic}
  \end{minipage}\hfill
  \begin{minipage}[t]{0.48\linewidth}
    \centering\textbf{STAB-SS}
    \begin{algorithmic}
      \FOR{$m = 1, \ldots, M$}
        \STATE \textbf{Fit base learner on training data:}\\
               \quad $\hat{h}_m(x) \leftarrow
                 \mathrm{Tree}\!\left(
                   \nabla L\{\hat{f}_m(X_i), W_i\},\; i \in \mathcal{D}_\mathrm{tr}
                 \right)$
        \IF{$m \geq 2$}
          \STATE \textbf{Evaluate multi-direction statistic:}\\
                  \quad $\mathbf{h}_m \leftarrow (\hat{h}_m,\; \hat h_{m-1})$\\
                  \quad $T \leftarrow T_n(\mathbf{h}_m,\, \hat{f}_m)$
          \IF{$T \leq c_\alpha$}
            \RETURN $\hat{f}_m$
          \ENDIF
        \ENDIF
      \STATE $\hat{f}_{m+1}(x) \leftarrow \hat{f}_m(x) - \eta_m\,\hat{h}_m(x)$
      \ENDFOR
      \RETURN $\hat{f}_M$
    \end{algorithmic}
  \end{minipage}
\end{algorithm}

\subsection{Setup}
\label{app:setup}

All synthetic experiments use LightGBM \citep{ke_lightgbm_2017} with \texttt{num\_leaves}~$=31$ and \texttt{min\_data\_in\_leaf}~$=20$.
The default learning rate is $\eta=0.05$; Appendix~\ref{app:lr-sensitivity} varies this value.
The four tasks are regression, binary classification, LambdaRank learning-to-rank, and Cox proportional-hazards survival.
The data-generating processes are given in Appendix~\ref{app:dgp}.

The regression, classification, and survival DGPs are run for $M=5{,}000$ boosting rounds and the ranking DGP for $M=2{,}000$ rounds.
For ranking we use LambdaRank with $\sigma = 1$, \texttt{lambdarank\_norm}~$=$~True, and \texttt{lambdarank\_truncation\_level}~$=50$.

We compare ScoreStop variants (FWD-SS, BWD-SS, STAB-SS) at thresholds $z \in \{0.025, 0.05, 0.1, 0.2, 0.3\}$ against loss-based stopping with patience $P \in \{1, 3, 5, 20, 50, 100\}$, where $P=1$ stops at the first non-improving validation-loss iteration.
The ScoreStop threshold is the $\chi^2_d$-calibrated critical value $c_\alpha = F^{-1}_{\chi^2_d}(1 - \alpha)$ corresponding to the two-sided $z$-score $z = \Phi^{-1}(1 - \alpha/2)$, with $d = 1$ for FWD-SS / BWD-SS and $d = 2$ for STAB-SS.
Numerical values for the thresholds are given in Table~\ref{tab:threshold-mapping}.
For loss-based early stopping, we report the model with the best validation loss during training, as is standard practice.
Because the population minimizer $f_0$ is known in synthetic experiments, we report excess test loss against the empirical test loss of $f_0$ on the same test sample (i.e., the test loss the population minimizer would itself attain); lower is better.
Each table reports the median across $100$ Monte Carlo seeds; for each seed we draw independent training, validation, and test sets.

\begin{table}[tb]
  \centering
  \caption{ScoreStop threshold scale.  The algorithm stops on the lower-tail event $T_n \le c_\alpha(z,d)$, where $\alpha(z)=2\{1-\Phi(z)\}$ and $c_\alpha(z,d)=F_{\chi^2_d}^{-1}\{1-\alpha(z)\}$.  The column $p_0 = 1 - \alpha$ is the probability under the null that a single statistic falls below $c_\alpha$; small $p_0$ means the threshold is conservative.}
  \label{tab:threshold-mapping}
  \small
  \begin{tabular}{rrrrr}
    \toprule
    $z$ & $\alpha(z)$ & $p_0$ & $c_\alpha(d=1)$ & $c_\alpha(d=2)$ \\
    \midrule
    0.025 & 0.980 & 0.020 & 0.000625 & 0.04029 \\
    0.05 & 0.960 & 0.040 & 0.0025 & 0.08139 \\
    0.1 & 0.920 & 0.080 & 0.01 & 0.166 \\
    0.2 & 0.841 & 0.159 & 0.04 & 0.3452 \\
    0.3 & 0.764 & 0.236 & 0.09 & 0.5379 \\
    1.96 & 0.050 & 0.950 & 3.842 & 5.992 \\
    \bottomrule
  \end{tabular}
\end{table}

\subsection{Synthetic data-generating processes}
\label{app:dgp}

We consider the following data-generating processes.

\textbf{Regression.}
$Y = f_0(X) + \varepsilon$, with
\begin{equation}
  f_0(X) = X^\top\beta
    + X_1 X_2
    + \mathbf{1}(X_4 > 0.5)\,X_2
    + \mathbf{1}(X_4 \leq 0.5)\,X_5,
  \label{eq:dgp}
\end{equation}
$X \sim \mathrm{Uniform}[0,1]^{10}$,
$\beta \sim \mathrm{Uniform}[0,1]^{10}$,
$\varepsilon \sim \mathcal{N}(0,1)$.
For each Monte Carlo seed a single draw of $\beta$ is fixed and shared
across the training, validation, and test splits so that the target
function $f_0$ is the same in all three; only the feature values $X$
and noise $\varepsilon$ are independently resampled.
Sample sizes (used throughout, except where noted):
$n_\mathrm{tr} = 2000$, $n_\mathrm{val} = 500$, and $n_\mathrm{te} = 10{,}000$.
This DGP combines linear effects, an interaction, and a threshold
non-linearity of the kind that gradient boosted trees model well.

\textbf{Classification.}
Adapted from the sparse additive model studied by \citet{wei_early_2019}:
\begin{equation}
  f_0(X) \;=\; 3\sin(2\pi X_1)
             + 3(2X_2 - 1)
             + 3(2X_3 - 1)(2X_4 - 1),
  \label{eq:binary-dgp}
\end{equation}
with $X \sim \mathrm{Uniform}[0,1]^{10}$ and
$P(Y = 1 \mid X) = \sigma(f_0(X))$, $\sigma(u) = (1+e^{-u})^{-1}$.
Sample sizes match the regression DGP.

\textbf{Survival.}
$f_0(X)$ is as in the regression DGP (Eq.~\ref{eq:dgp}) but interpreted
as a log-hazard ratio.  Failure times follow a Cox model with constant
baseline hazard $\Lambda_0(t) = t$:
\begin{equation}
  T^* \mid X \;\sim\; \mathrm{Exp}(\exp(f_0(X))),
  \qquad C \mid X \;\sim\; \mathrm{Exp}(\lambda_C),
  \label{eq:survival-dgp}
\end{equation}
with $C$ independent of $T^*$ and $\lambda_C$ chosen so the population
censoring rate is $30\%$
($\lambda_C = (3/7)\,\mathrm{median}(\exp f_0(X))$).
The observed data is $W = (D, T, X)$ with $T = \min(T^*, C)$ and
$D = \mathbf{1}\{T^* \leq C\}$.

\textbf{Ranking.}
Each query has a uniformly random number of documents in $[10, 40]$,
each with predictors $X \sim \mathrm{Uniform}[0,1]^{10}$.
The latent relevance is the sparse non-linear function
\begin{equation}
  f_0(X) = 2\sin(2\pi X_1) + 1.5\,X_2 X_3 + X_4 + \varepsilon,
  \label{eq:ranking-dgp}
\end{equation}
with $\varepsilon \sim \mathcal{N}(0, 0.5^2)$.
Integer relevance labels in $\{0, 1, 2, 3, 4\}$ are obtained by
digitising $f_0$ at the population quantiles $\{30, 55, 75, 90\}$,
giving a skewed label distribution with progressively fewer
high-relevance documents.
Sample sizes are matched to the per-task $n_\mathrm{tr} = 2000$,
$n_\mathrm{val} = 500$, and $n_\mathrm{te} = 10{,}000$ in
\emph{documents}; with documents per query uniform in $[10, 40]$
(mean $25$) this corresponds to $80$, $20$, and $400$ queries
respectively.

\subsection{Results at the default learning rate}
\label{app:fixed-lr-results}

Table~\ref{tab:synthetic-combined-main} gives the numeric values summarized in Figure~\ref{fig:synthetic-summary}, including BWD-SS and the selected iteration counts.

\begin{table*}[tb]
  \centering
  \caption{Synthetic experiments at $\eta=0.05$. Entries are median excess test loss over the known population minimizer $f_0$ and median selected iteration across 100 seeds. ScoreStop uses the $\chi^2_d$ calibration with $d=1$ for FWD-SS/BWD-SS and $d=2$ for STAB-SS.}
  \label{tab:synthetic-combined-main}
  \small
  \begin{tabular}{lrrrrrrrr}
    \toprule
     & \multicolumn{2}{c}{\textbf{Regression}} & \multicolumn{2}{c}{\textbf{Classification}} & \multicolumn{2}{c}{\textbf{Ranking}} & \multicolumn{2}{c}{\textbf{Survival}} \\
    \cmidrule(lr){2-3} \cmidrule(lr){4-5} \cmidrule(lr){6-7} \cmidrule(lr){8-9}
    \textbf{Method} & \textbf{Excess} & \textbf{Iter} & \textbf{Excess} & \textbf{Iter} & \textbf{Excess} & \textbf{Iter} & \textbf{Excess} & \textbf{Iter} \\
    \midrule
    FWD-SS, $z=0.025$ & 0.049 & 77 & 0.0636 & 92 & 0.0155 & 99 & 0.0795 & 76 \\
    FWD-SS, $z=0.05$ & 0.0482 & 64 & 0.0633 & 79 & 0.0155 & 82 & 0.0688 & 50 \\
    FWD-SS, $z=0.1$ & 0.0485 & 53 & 0.0636 & 58 & 0.0155 & 70 & 0.0645 & 39 \\
    FWD-SS, $z=0.2$ & 0.0498 & 46 & 0.0672 & 51 & 0.0159 & 64 & 0.0658 & 31 \\
    FWD-SS, $z=0.3$ & 0.0509 & 44 & 0.0679 & 49 & 0.0161 & 58 & 0.0676 & 29 \\
    BWD-SS, $z=0.025$ & 0.051 & 67 & 0.0633 & 90 & 0.0153 & 97 & 0.075 & 62 \\
    BWD-SS, $z=0.05$ & 0.0495 & 57 & 0.0625 & 72 & 0.0155 & 82 & 0.0669 & 44 \\
    BWD-SS, $z=0.1$ & 0.0505 & 47 & 0.0645 & 57 & 0.0155 & 73 & 0.0663 & 35 \\
    BWD-SS, $z=0.2$ & 0.0528 & 43 & 0.0678 & 50 & 0.0156 & 63 & 0.0681 & 28 \\
    BWD-SS, $z=0.3$ & 0.0544 & 40 & 0.0686 & 48 & 0.0157 & 58 & 0.0704 & 26 \\
    STAB-SS, $z=0.025$ & 0.0519 & 95 & 0.0633 & 112 & 0.0154 & 129 & 0.0859 & 82 \\
    STAB-SS, $z=0.05$ & 0.0486 & 67 & 0.0631 & 84 & 0.0152 & 99 & 0.066 & 47 \\
    STAB-SS, $z=0.1$ & 0.0477 & 59 & 0.0637 & 67 & 0.0152 & 83 & 0.0632 & 41 \\
    STAB-SS, $z=0.2$ & 0.0484 & 50 & 0.0645 & 57 & 0.0153 & 73 & 0.0636 & 34 \\
    STAB-SS, $z=0.3$ & 0.0492 & 49 & 0.0656 & 54 & 0.0155 & 68 & 0.0649 & 32 \\
    Patience-1 & 0.0513 & 43 & 0.0687 & 47 & 0.0322 & 3 & 0.0675 & 29 \\
    Patience-3 & 0.0474 & 54 & 0.0627 & 64 & 0.0257 & 5 & 0.0626 & 38 \\
    Patience-5 & 0.0465 & 57 & 0.0615 & 69 & 0.0236 & 9 & 0.0622 & 40 \\
    Patience-20 & 0.0473 & 65 & 0.061 & 91 & 0.0196 & 24 & 0.0617 & 43 \\
    Patience-50 & 0.0473 & 66 & 0.0607 & 98 & 0.0167 & 46 & 0.0617 & 43 \\
    Patience-100 & 0.0474 & 66 & 0.0607 & 99 & 0.0167 & 60 & 0.0617 & 43 \\
    \bottomrule
  \end{tabular}
\end{table*}

\subsection{Null calibration}
\label{app:null-calibration}

We check the finite-sample calibration of the score statistics empirically using the synthetic boosting trajectories of Appendix~\ref{app:dgp}.
For each task we identify a single \emph{trigger iteration} $m^*$, equal to the median, across the $100$ Monte Carlo seeds, of the iteration with the minimum test loss.
This procedure gives a single trigger iteration per task, which is close to the population minimizer $f_0$.
The trigger iterations for each task are: regression ($m^*=63$), classification ($m^*=90$), ranking ($m^*=102$), and survival ($m^*=42$).
We then pool the score statistics from a $\pm 5$ iteration window around $m^*$ across all seeds, giving $11 \times 100 = 1{,}100$ score statistic values per task and ScoreStop variant.
These are plotted in QQ plots against the $\chi^2_d$ reference distribution in Figure~\ref{fig:null-calibration}.
These plots confirm that the distribution of the ScoreStop statistic is well-approximated by the $\chi^2_d$ reference distribution, even in finite samples, when the function $\hat f_m$ is close to the population minimizer $f_0$.

\begin{figure*}[tb]
  \centering
  \includegraphics[width=\textwidth]{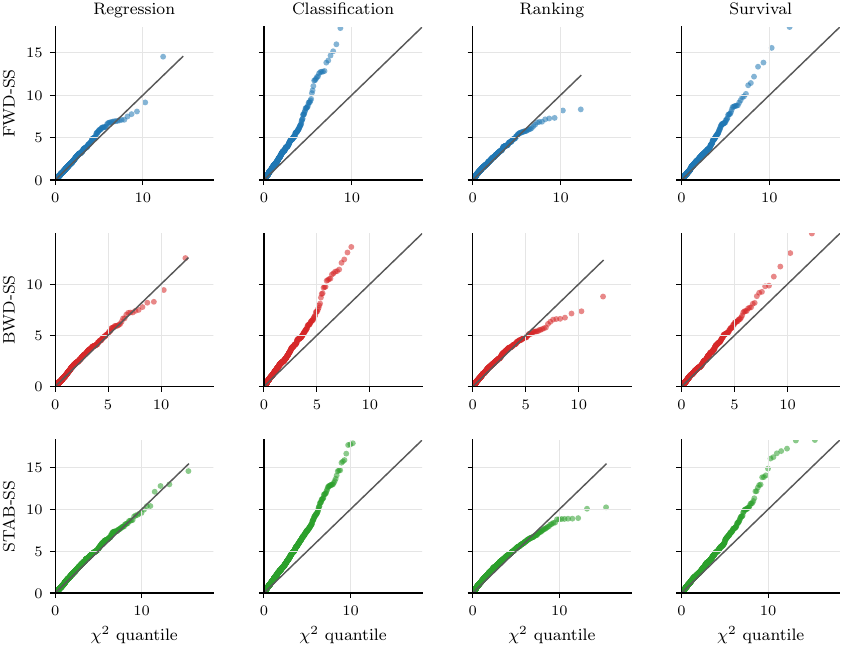}
  \caption{QQ plots of the score statistics in a $\pm 5$ iteration
    window around the median test-loss argmin iteration, for each task.
    That is, each plot contains eleven points for each of the $100$ Monte Carlo seeds.
    The degrees of freedom $d$ of the reference $\chi^2_d$ distribution is given by the ScoreStop variant:
    FWD-SS, BWD-SS ($d=1$) and STAB-SS ($d=2$).
    The diagonal line on each plot represents perfect agreement with the asymptotic null reference.}
  \label{fig:null-calibration}
\end{figure*}

\subsection{Learning rate sensitivity}
\label{app:lr-sensitivity}

Table~\ref{tab:lr-sensitivity} reports the median excess test loss of
FWD-SS, BWD-SS, and STAB-SS at $z \in \{0.05, 0.1\}$, alongside
Patience at $P \in \{5, 20, 50\}$, across
$\eta \in \{0.5, 0.1, 0.05, 0.01\}$ for all four synthetic DGPs.
These results separate learning-rate effects from threshold effects: all methods are less stable at the largest learning rate, while the default $\eta=0.05$ lies in the range where the ScoreStop variants and patience baselines have comparable excess losses.

\begin{table*}[tb]
  \centering
  \caption{Learning-rate sensitivity. Entries are median excess test loss over the known population minimizer $f_0$ across 100 seeds. ScoreStop uses $z \in \{0.05,0.1\}$; patience uses $P \in \{5,20,50\}$.}
  \label{tab:lr-sensitivity}
  \small
  \begin{tabular}{lrrrr}
    \toprule
    \textbf{Method} & \textbf{$\eta = 0.5$} & \textbf{$\eta = 0.1$} & \textbf{$\eta = 0.05$} & \textbf{$\eta = 0.01$} \\
    \midrule
    \multicolumn{5}{l}{\textit{Regression}} \\
    \midrule
    FWD-SS, $z=0.05$ & 0.147 & 0.0537 & 0.0482 & 0.0505 \\
    FWD-SS, $z=0.1$ & 0.115 & 0.0531 & 0.0485 & 0.0531 \\
    BWD-SS, $z=0.05$ & 0.188 & 0.0544 & 0.0495 & 0.0509 \\
    BWD-SS, $z=0.1$ & 0.171 & 0.0524 & 0.0505 & 0.0535 \\
    STAB-SS, $z=0.05$ & 0.187 & 0.0547 & 0.0486 & 0.0486 \\
    STAB-SS, $z=0.1$ & 0.174 & 0.0523 & 0.0477 & 0.0496 \\
    Patience-5 & 0.0852 & 0.0495 & 0.0465 & 0.0467 \\
    Patience-20 & 0.0852 & 0.0503 & 0.0473 & 0.0454 \\
    Patience-50 & 0.0852 & 0.0503 & 0.0473 & 0.0455 \\
    \midrule
    \multicolumn{5}{l}{\textit{Classification}} \\
    \midrule
    FWD-SS, $z=0.05$ & 0.172 & 0.0647 & 0.0633 & 0.0667 \\
    FWD-SS, $z=0.1$ & 0.126 & 0.0633 & 0.0636 & 0.0686 \\
    BWD-SS, $z=0.05$ & 0.35 & 0.0653 & 0.0625 & 0.0674 \\
    BWD-SS, $z=0.1$ & 0.172 & 0.0639 & 0.0645 & 0.0689 \\
    STAB-SS, $z=0.05$ & 0.59 & 0.0649 & 0.0631 & 0.0639 \\
    STAB-SS, $z=0.1$ & 0.233 & 0.0635 & 0.0637 & 0.066 \\
    Patience-5 & 0.0857 & 0.0623 & 0.0615 & 0.0632 \\
    Patience-20 & 0.0859 & 0.0615 & 0.061 & 0.0608 \\
    Patience-50 & 0.0859 & 0.0616 & 0.0607 & 0.0603 \\
    \midrule
    \multicolumn{5}{l}{\textit{Ranking}} \\
    \midrule
    FWD-SS, $z=0.05$ & 0.0257 & 0.0167 & 0.0155 & 0.0152 \\
    FWD-SS, $z=0.1$ & 0.0293 & 0.0164 & 0.0155 & 0.0156 \\
    BWD-SS, $z=0.05$ & 0.0285 & 0.016 & 0.0155 & 0.0156 \\
    BWD-SS, $z=0.1$ & 0.0312 & 0.0163 & 0.0155 & 0.0155 \\
    STAB-SS, $z=0.05$ & 0.0244 & 0.0161 & 0.0152 & 0.0151 \\
    STAB-SS, $z=0.1$ & 0.0278 & 0.0161 & 0.0152 & 0.0153 \\
    Patience-5 & 0.0341 & 0.0216 & 0.0236 & 0.0291 \\
    Patience-20 & 0.0263 & 0.0179 & 0.0196 & 0.0252 \\
    Patience-50 & 0.0232 & 0.0173 & 0.0167 & 0.023 \\
    \midrule
    \multicolumn{5}{l}{\textit{Survival}} \\
    \midrule
    FWD-SS, $z=0.05$ & 0.617 & 0.0804 & 0.0688 & 0.0652 \\
    FWD-SS, $z=0.1$ & 0.419 & 0.0755 & 0.0645 & 0.0688 \\
    BWD-SS, $z=0.05$ & 1.15 & 0.0821 & 0.0669 & 0.0659 \\
    BWD-SS, $z=0.1$ & 0.873 & 0.0767 & 0.0663 & 0.0677 \\
    STAB-SS, $z=0.05$ & 1.83 & 0.0842 & 0.066 & 0.0614 \\
    STAB-SS, $z=0.1$ & 1.14 & 0.0751 & 0.0632 & 0.0638 \\
    Patience-5 & 0.186 & 0.068 & 0.0622 & 0.0597 \\
    Patience-20 & 0.186 & 0.0683 & 0.0617 & 0.0595 \\
    Patience-50 & 0.186 & 0.0683 & 0.0617 & 0.0598 \\
    \bottomrule
  \end{tabular}
\end{table*}

\subsection{Real data numeric results}
\label{app:real-data-results}

Table~\ref{tab:real_data} gives the fold-median values summarized in Figure~\ref{fig:real-data-summary}.

\begin{table*}[!t]
  \centering
  \caption{Real-data benchmarks. \textbf{Oracle} is the median minimum test loss along the fitted trajectory; method columns report median excess over that oracle, so lower is better. ScoreStop uses $z=0.05$ throughout. The ranking loss is $1-\mathrm{NDCG}@10$.}
  \label{tab:real_data}
  \footnotesize
  \begin{tabular}{llrrrrrrr}
    \toprule
    \textbf{Dataset} & \textbf{Task} & $n$ & \textbf{Oracle} & \textbf{FWD-SS} & \textbf{BWD-SS} & \textbf{STAB-SS} & \textbf{Pat-5} & \textbf{Pat-20} \\
    \midrule
    California Housing & Regression & 20,640 & 0.440 & 0.0234 & 0.0208 & 0.0174 & 0.0122 & 0.0038 \\
    Abalone & Regression & 4,177 & 2.130 & 0.0102 & 0.0049 & 0.0060 & 0.0075 & 0.0075 \\
    CPU Activity & Regression & 8,192 & 2.246 & 0.0517 & 0.0690 & 0.0363 & 0.0296 & 0.0127 \\
    Adult Income & Classification & 48,842 & 0.278 & 0.0026 & 0.0025 & 0.0014 & 0.0012 & 0.0003 \\
    Phoneme & Classification & 5,404 & 0.254 & 0.0046 & 0.0172 & 0.0082 & 0.0058 & 0.0042 \\
    German Credit & Classification & 1,000 & 0.498 & 0.0106 & 0.0096 & 0.0105 & 0.0051 & 0.0085 \\
    freMTPL2 & Poisson & 678,013 & 0.192 & 0.0007 & 0.0008 & 0.0006 & 0.0005 & 0.0002 \\
    Bike Sharing & Poisson & 17,379 & -878.646 & 0.4614 & 0.5026 & 0.2621 & 0.2803 & 0.1055 \\
    Covertype & Multiclass (7) & 581,012 & 0.098 & 0.0317 & 0.0364 & 0.0183 & 0.0119 & 0.0023 \\
    MEPS & Quantile ($\tau{=}0.9$) & 28,336 & 3497.428 & 5.3082 & 3.7547 & 4.6194 & 3.6818 & 3.1127 \\
    MSLR-WEB10K & Ranking & 10,000 & 0.494 & 0.0052 & 0.0056 & 0.0043 & 0.0092 & 0.0048 \\
    FLCHAIN & Survival (Cox) & 4,000 & 5.069 & 0.0130 & 0.0412 & 0.0205 & 0.0104 & 0.0133 \\
    SUPPORT & Survival (Cox) & 9,105 & 5.073 & 0.0082 & 0.0143 & 0.0064 & 0.0093 & 0.0051 \\
    \bottomrule
  \end{tabular}
\end{table*}

\section{Proofs of Theoretical Results}
\label{app:proofs}

We prove the multi-direction generalizations of Propositions~\ref{prop:null} and~\ref{prop:power}, from which the single-direction results follow as
special cases.
Throughout, we work with a fixed direction set $\mathbf{h} = (h_1, \ldots, h_d) \in \mathcal H^d$ with $d \ge 1$, and we write the vectors
\begin{align*}
    \boldsymbol\varphi_0(\mathbf{h}, f) &=
        \big(\varphi_0(h_1, f), \ldots, \varphi_0(h_d, f)\big), \\
    \hat{\boldsymbol\varphi}(\mathbf{h}, f) &=
        \big(\hat\varphi(h_1, f), \ldots, \hat\varphi(h_d, f)\big).
\end{align*}
The multi-direction ScoreStop statistic is
\begin{align*}
    T_n(\mathbf{h}, f_*) = n\,
    \Env[\hat{\boldsymbol\varphi}(\mathbf{h}, f_*)]^\top
    \hat{\Sigma}_n(\mathbf{h}, f_*)^{-1}
    \Env[\hat{\boldsymbol\varphi}(\mathbf{h}, f_*)],
\end{align*}
where we define the covariance estimator
\begin{align*}
    \hat{\Sigma}_n(\mathbf{h}, f_*) = \Env[\hat{\boldsymbol\varphi}(\mathbf{h}, f_*)\hat{\boldsymbol\varphi}(\mathbf{h}, f_*)^\top].
\end{align*}
This expression recovers \eqref{eq:tstat_if} when $d=1$.
We note that $T_n(\mathbf{h}, f_*)$ is closely related to the $t^2$ statistic of \citet{hotelling_generalization_1931}, which is a multivariate generalization of Student's $t^2$ statistic.

\begin{definition}[Hotelling's $t^2$ statistic]
    Suppose that one has access to $n$ iid observations of a random vector $X$, with mean $\mu$.
    Hotelling's $t^2$ statistic is $t_n = n(\bar{x}-\mu)^\top\Sigma_n^{-1}(\bar{x}-\mu)$, where we define the sample mean $\bar{x} = \Env[X]$ and the sample covariance
    \begin{align*}
        \Sigma_n = \frac{n}{n-1}\Env[(X-\bar{x})(X-\bar{x})^\top].
    \end{align*}
\end{definition}

Our statistic can be viewed as a Hotelling's statistic with a mean value $\mu = 0$ and using a covariance estimator $\hat{\Sigma}_n(\mathbf{h}, f_*)$ that is centered on $\mu = 0$ rather than the sample mean $\Env[\hat{\boldsymbol\varphi}(\mathbf{h}, f_*)]$. Additionally, our analysis must account for the fact that we average over estimated influence functions $\hat{\boldsymbol\varphi}(\mathbf{h}, f_*)$ rather than the true influence functions $\boldsymbol\varphi_0(\mathbf{h}, f_*)$, which is an important distinction only when the loss is data dependent.

\subsection{Multi-direction assumptions}

We assume the multi-direction analogs of Assumptions~\ref{assump:central}, \ref{assump:direction}, \ref{assump:variance}, and~\ref{assump:ral}:

\begin{enumerate}
    \item[(A1)] $\E[\boldsymbol\varphi_0(\mathbf{h}, f_0)] = \mathbf{0}$.
    \item[(A2)] The map $h \mapsto \E[\varphi_0(h, f_*)]$ is a
        continuous linear functional with Riesz representer $h^*(f_*)$.
    \item[(A3)] The matrix $\Sigma(\mathbf{h}, f_0)$ is finite and invertible, where $\Sigma(\mathbf{h}, f) = \E[\boldsymbol\varphi_0(\mathbf{h}, f)
        \boldsymbol\varphi_0(\mathbf{h}, f)^\top]$.
    \item[(A4)] Regular asymptotic linearity, $\sqrt n\, \Env[\hat{\boldsymbol\varphi}(\mathbf{h}, f_0) -
        \boldsymbol\varphi_0(\mathbf{h}, f_0)] = o_p(1)$.
    \item[(A5)] First-order consistency: $\Env[\hat{\boldsymbol\varphi}(\mathbf{h}, f_*) - \boldsymbol\varphi_0(\mathbf{h}, f_*)] = o_p(1)$.
    \item[(A6)] Second-order consistency: $\Env[\hat{\boldsymbol\varphi}(\mathbf{h}, f_*)\hat{\boldsymbol\varphi}(\mathbf{h}, f_*)^\top - \boldsymbol\varphi_0(\mathbf{h}, f_*)
        \boldsymbol\varphi_0(\mathbf{h}, f_*)^\top] = o_p(1)$.
\end{enumerate}

\subsection{Behavior under the null}

To analyze the ScoreStop statistic, we define the centered, normalized vector
\begin{align*}
    Z_n(\mathbf{h}, f_*) = \sqrt n\,
    \big\{\Env[\hat{\boldsymbol\varphi}(\mathbf{h}, f_*)] -
        \E[\boldsymbol\varphi_0(\mathbf{h}, f_*)]\big\},
\end{align*}
and hence,
\begin{align}
    T_n(\mathbf{h}, f_*) &=
    \left\{Z_n(\mathbf{h}, f_*) + \sqrt n\,\E[\boldsymbol\varphi_0(\mathbf{h}, f_*)]\right\}^\top
    \hat\Sigma_n(\mathbf{h}, f_*)^{-1}
    \left\{Z_n(\mathbf{h}, f_*) + \sqrt n\,\E[\boldsymbol\varphi_0(\mathbf{h}, f_*)]\right\}.
    \label{eq:hotelling-decomp}
\end{align}
We first prove two lemmas on the asymptotic normality of $Z_n(\mathbf{h}, f_*)$ and the estimation of its covariance.

\begin{lemma}[CLT]\label{lem:joint-clt}
Under (A1), (A3) and (A4), and under the null $H_0: f_* = f_0$,
\begin{align*}
    Z_n(\mathbf{h}, f_*) \overset{d}{\to}
    \mathcal N_d\big(\mathbf{0}, \Sigma(\mathbf{h}, f_0)\big)
    \quad \text{as } n \to \infty.
\end{align*}
\end{lemma}

\begin{proof}
Omitting the input arguments $(\mathbf{h}, f_*)$ we decompose
\begin{align*}
    Z_n(\mathbf{h}, f_*) &=
    \sqrt n\,\big\{\Env[\hat{\boldsymbol\varphi}] -
        \Env[\boldsymbol\varphi_0]\big\}
    + \sqrt n\,\big\{\Env[\boldsymbol\varphi_0] -
        \E[\boldsymbol\varphi_0]\big\}.
\end{align*}
The first term is $o_p(1)$ by (A4).
The second term converges in distribution to $\mathcal N_d(\mathbf{0}, \Sigma(\mathbf{h}, f_0))$ by the multivariate CLT, centered on $\bm{0}$ by (A1), and with finite covariance matrix by (A3).
Slutsky's theorem allows for the separate limit results for the first and second terms to be added.
\end{proof}

\begin{lemma}[Precision consistency]\label{lem:cov-cons}
Assume (A6). If $\Sigma(\mathbf{h}, f_*)$ is finite and invertible, then $\hat\Sigma(\mathbf{h}, f_*)^{-1} - \Sigma(\mathbf{h}, f_*)^{-1} = o_p(1)$.
\end{lemma}

\begin{proof}
Omitting the input arguments $(\mathbf{h}, f_*)$ we write
\begin{align*}
    \hat\Sigma_n(\mathbf{h}, f_*) - \Sigma(\mathbf{h}, f_*) &=
    \left\{\Env[\hat{\boldsymbol\varphi}\hat{\boldsymbol\varphi}^\top] -
        \Env[\boldsymbol\varphi_0\boldsymbol\varphi_0^\top]\right\}  + \left\{\Env[\boldsymbol\varphi_0\boldsymbol\varphi_0^\top] -
        \E[\boldsymbol\varphi_0\boldsymbol\varphi_0^\top]\right\}.
\end{align*}
The first term is $o_p(1)$ by (A6) and the second
term is $o_p(1)$ by the law of large numbers, using (A3). These are combined to give $\hat\Sigma_n - \Sigma = o_p(1)$. The result follows by the continuous mapping theorem applied to matrix inversion.
\end{proof}

Using these Lemmas, we prove the multi-direction analog of Proposition~\ref{prop:null}, with the single-direction result in Proposition~\ref{prop:null} following when $d = 1$.

\begin{proposition}[Multi-direction null distribution]\label{prop:null-multi}
Under (A1)--(A6), and the null hypothesis $H_0: f_* = f_0$, $T_n(\mathbf{h}, f_*)
\overset{d}{\to} \chi^2_d$ as $n \to \infty$.
\end{proposition}

\begin{proof}
Under the null hypothesis $H_0: f_*=f_0$ then $\E[\boldsymbol\varphi_0(\mathbf{h}, f_*)] = \bm{0}$ by (A1), and \eqref{eq:hotelling-decomp} reduces to
\begin{align*}
    T_n &= Z_n^\top \hat\Sigma_n^{-1} Z_n \\
    &=Z_n^\top \Sigma^{-1} Z_n + Z_n^\top \{\hat\Sigma_n^{-1} - \Sigma^{-1}\} Z_n
\end{align*}
By Lemma~\ref{lem:joint-clt}, $Z_n \overset{d}{\to} \mathcal N_d(\mathbf{0}, \Sigma)$, and hence, applying (A3) and the continuous mapping theorem, the first term has asymptotic distribution $Z_n^\top \Sigma^{-1} Z_n \overset{d}{\to} \chi^2_d$.
For the second term, $Z_n = O_p(1)$ by tightness of the convergent sequence,
and $\hat\Sigma_n^{-1} - \Sigma^{-1} = o_p(1)$ by Lemma~\ref{lem:cov-cons}, and assuming (A3).
Then the second term $Z_n^\top \{\hat\Sigma_n^{-1} - \Sigma^{-1}\} Z_n = o_p(1)$. Combining the results for the first and second terms completes the proof.
\end{proof}

\subsection{Behavior under the alternative}

The following proposition shows that tests based on the ScoreStop statistic have the power to test against certain alternative hypotheses that depend on the direction of the test.
Proposition~\ref{prop:power} follows from the case $d = 1$.

\begin{proposition}[Multi-direction power consistency]\label{prop:power-multi}
Fix $\mathbf{h} \in \mathcal H^d$ and assume (A2), (A5) and (A6).
If $\inner{h_j}{h^*(f_*)} \neq 0$ for at least one direction $j$, and $\Sigma(\mathbf{h}, f_*)$ is finite and invertible, then under the alternative hypothesis $H_1: f_* \ne f_0$, $P\{T_n(\mathbf{h}, f_*) > c_\alpha\} \to 1$ for any
fixed $c_\alpha$.
\end{proposition}

\begin{proof}
Let $\mu = \E[\boldsymbol\varphi_0(\mathbf{h}, f_*)]$, and note that, by (A2), the vector $\mu$ has $j$th component $\mu_j = \inner{h_j}{h^*(f_*)}$. Hence, $\mu \neq 0$.
From~\eqref{eq:hotelling-decomp},
\begin{align*}
    \frac{1}{n} T_n(\mathbf{h}, f_*) &=
    \big\{n^{-1/2} Z_n + \mu\big\}^\top
    \hat\Sigma_n^{-1}
    \big\{n^{-1/2} Z_n + \mu\big\} \\
    &= \big\{n^{-1/2} Z_n + \mu\big\}^\top
    \Sigma^{-1}
    \big\{n^{-1/2} Z_n + \mu\big\} + \big\{n^{-1/2} Z_n + \mu\big\}^\top
    \{\hat\Sigma_n^{-1} - \Sigma^{-1}\}
    \big\{n^{-1/2} Z_n + \mu\big\}.
\end{align*}
Note that
\begin{align*}
    n^{-1/2} Z_n =     \left\{\Env[\hat{\boldsymbol\varphi}] -
        \Env[\boldsymbol\varphi_0]\right\}
    + \left\{\Env[\boldsymbol\varphi_0] -
        \E[\boldsymbol\varphi_0]\right\}.
\end{align*}
where the first term is $o_p(1)$ by (A5) and the second is $o_p(1)$ by the law of large numbers.
Next, note that $\hat\Sigma_n - \Sigma = o_p(1)$ by Lemma~\ref{lem:cov-cons}.
Therefore,
\begin{align*}
    \frac{1}{n} T_n(\mathbf{h}, f_*) \overset{p}{\to}
    \mu^\top \Sigma(\mathbf{h}, f_*)^{-1} \mu > 0,
\end{align*}
where positivity follows from $\mu \neq 0$ and positive-definiteness of
$\Sigma^{-1}$. Hence $T_n(\mathbf{h}, f_*) \overset{p}{\to} \infty$, and
$P\{T_n > c_\alpha\} \to 1$ for any fixed $c_\alpha$.
\end{proof}

\begin{proof}[Proof of Corollary~\ref{corol:opt}]
The Cauchy--Schwarz inequality gives $|\inner{h}{h^*(f_*)}| \leq \|h\|_2\,\|h^*(f_*)\|_2$, with equality iff $h \propto h^*(f_*)$. When $h = h^*(f_*)$ and $h^*(f_*) \ne 0$, the condition $\langle h, h^*(f_*)\rangle = \|h^*(f_*)\|^2 \ne 0$ of Proposition~\ref{prop:power} is satisfied.
\end{proof}

\section{LambdaRank gradients and the NDCG target loss}
\label{app:lambdarank}

This appendix gives explicit expressions for the LambdaRank gradients
$\lambda_k$ and the underlying nonsmooth target loss based on NDCG.
We follow the formulation of \citet{burges_ranknet_2010}.

\subsection{The NDCG target loss}

For a query $W = (N, Y, X)$ with document scores
$\underline{f}_N(X) = (f(X_1), \ldots, f(X_N))$, let $\pi_f$ denote the
permutation of $\{1, \ldots, N\}$ that sorts the documents by score in
decreasing order, so that $f(X_{\pi_f(1)}) \geq \cdots \geq f(X_{\pi_f(N)})$.
The Discounted Cumulative Gain at truncation level $T \leq N$ is
\begin{align*}
    \mathrm{DCG}_T\{\underline{f}_N(X), W\}
    = \sum_{r=1}^{T} \frac{2^{Y_{\pi_f(r)}} - 1}{\log_2(1 + r)},
\end{align*}
and the Normalized DCG is
\begin{align*}
    \mathrm{NDCG}_T\{\underline{f}_N(X), W\}
    = \frac{\mathrm{DCG}_T\{\underline{f}_N(X), W\}}{\mathrm{DCG}_T^\star(W)},
\end{align*}
where $\mathrm{DCG}_T^\star(W)$ is the maximum DCG attainable for the query, obtained by sorting documents by their labels $Y_k$ rather than their scores.
The target nonsmooth loss is
\begin{align*}
    L^\star\{\underline{f}_N(X), W\}
    = 1 - \mathrm{NDCG}_T\{\underline{f}_N(X), W\}
    \in [0, 1].
\end{align*}
Because $\mathrm{NDCG}_T$ depends on $\underline{f}_N(X)$ only through the permutation $\pi_f$, the loss $L^\star$ is piecewise constant in the document scores, with jumps along the hyperplanes $\{f(X_i) = f(X_j)\}$ where the sort order changes.

\subsection{LambdaRank gradients}

Let $\mathcal{I}(W) = \{(i,j) : Y_i > Y_j\}$ denote the set of ordered document pairs for which document $i$ is more relevant than document $j$.
For a pair $(i,j) \in \mathcal{I}(W)$, define
\begin{align*}
    |\Delta \mathrm{NDCG}_{ij}\{\underline{f}_N(X), W\}|
    = \big|\mathrm{NDCG}_T\{\underline{f}_N(X), W\}
    - \mathrm{NDCG}_T\{\underline{f}_N(X)^{(ij)}, W\}\big|,
\end{align*}
where $\underline{f}_N(X)^{(ij)}$ denotes the score vector with the ranks of documents $i$ and $j$ swapped, and all other ranks fixed.
The LambdaRank pairwise contribution is
\begin{align*}
    \lambda_{ij}\{\underline{f}_N(X), W\}
    = \frac{-\sigma}{1 + \exp\{\sigma(f(X_i) - f(X_j))\}}
    \, |\Delta \mathrm{NDCG}_{ij}\{\underline{f}_N(X), W\}|,
\end{align*}
where $\sigma > 0$ is a shape parameter, typically $\sigma = 1$.
The LambdaRank gradient for document $k$ is then
\begin{align}
    \lambda_k\{\underline{f}_N(X), W\}
    = \sum_{j : (k,j) \in \mathcal{I}(W)}
      \lambda_{kj}\{\underline{f}_N(X), W\}
    - \sum_{j : (j,k) \in \mathcal{I}(W)}
      \lambda_{jk}\{\underline{f}_N(X), W\}.
    \label{eq:lambdarank_gradient}
\end{align}
The first sum collects contributions from pairs where document $k$ should be ranked above another document, and the second from pairs where it should be ranked below.

\subsection{Implicit loss and the Poincaré condition}

Fix a score vector $\underline{f}_N(X)$ with no ties, i.e.,
$f(X_i) \neq f(X_j)$ for all $i \neq j$. On a neighborhood
of such a point, the sort permutation $\pi_f$ is constant, hence the $|\Delta \mathrm{NDCG}_{ij}|$ values are constant. The function
\begin{align*}
    L\{\underline{f}_N(X), W\}
    = \sum_{(i,j) \in \mathcal{I}(W)}
      |\Delta \mathrm{NDCG}_{ij}\{\underline{f}_N(X), W\}|
      \log\big\{1 + \exp\{-\sigma(f(X_i) - f(X_j))\}\big\}
\end{align*}
is differentiable on this neighborhood and satisfies
$\partial L / \partial f(X_k) = \lambda_k\{\underline{f}_N(X), W\}$ for each $k$.
The Poincaré symmetry condition therefore holds for any
no-tie configuration, and there is a local implicit loss $L$ for the LambdaRank gradients.
Globally, $L$ is piecewise smooth, with the $|\Delta \mathrm{NDCG}_{ij}|$ coefficients changing across the tie hyperplanes.

\section{Survival analysis derivations}\label{app:survival}

This appendix gives explicit expressions related to the ScoreStop statistic for the Cox proportional hazards model.

\subsection{Cox Model}

We assume that the observed data $W = (D, T, X) \sim P_0$ is constructed from latent observations of $(T^*, C, X)$, where $T^* \in \mathbb{R}_+$ is the failure time that would be observed if there was no censoring, and $C \in \mathbb{R}_+$ is a random censoring time. The Cox model assumes non-informative censoring, in the sense that $C$ is conditionally independent of $T^*$ given $X$. The observed data is constructed as $T = \min(T^*, C)$ and $D = \mathbf{1}\{T^* \leq C\}$.

Under this assumption, the likelihood for a single observation $(\delta, t, x)$ is
\begin{align}
    p_{T^*|X}(t|x)^\delta S_{T^*|X}(t|x)^{1-\delta} \cdot p_{C|X}(t|x)^{1-\delta} S_{C|X}(t|x)^{\delta} \cdot p_X(x) \label{eq:cox_likelihood}
\end{align}
where $p$ denotes the density under $P_0$ and $S$ denote the survival function, e.g. $S_{T^*|X}(t|x) = P(T^*>t|X=x)$. The cumulative hazard is related to the survival function as $\Lambda(t) = -\log\{S(t)\}$, and since the density $p(t) = \mathrm{d}\{1 - S(t)\} / \mathrm{d}t$ is the derivative of the cumulative distribution function, $p(t) = \lambda(t) S(t)$, where $\lambda(t) = \mathrm{d}\Lambda(t) / \mathrm{d}t$ is called the hazard. Hence, the likelihood in \eqref{eq:cox_likelihood} is
\begin{align*}
    \lambda_{T^*|X}(t|x)^\delta S_{T^*|X}(t|x) \cdot \lambda_{C|X}(t|x)^{1-\delta} S_{C|X}(t|x) \cdot p_X(x)
\end{align*}
The Cox model is $\lambda_{T^*|X}(t|x) = \lambda_0(t) \exp\{f_0(x)\}$, where $\lambda_0(t) = \mathrm{d}\Lambda_0(t) / \mathrm{d}t$ is a baseline hazard, and $f_0(x)$ is a log-hazard ratio. Note $\Lambda_{T^*|X}(t|x) = \exp\{f_0(x)\}\Lambda_0(t)$.
Thus, the negative log-likelihood is
\begin{align*}
    \exp\{f_0(x)\}\Lambda_0(t) - \delta f_0(x) + c(\delta,t,x)
\end{align*}
where $c(\delta,t,x)$ is an additive function that does not depend on $f_0$. Discarding this function, the partial negative log-likelihood loss, and its derivative are
\begin{align*}
    L\{f(X), W; \Lambda_0\} &= \exp\{f(X)\}\Lambda_0(T) - D f(X) \\
    \nabla L\{f(X), W; \Lambda_0\}&= \exp\{f(X)\}\Lambda_0(T) - D,
\end{align*}
which both depend on $P_0$ through the baseline cumulative hazard $\Lambda_0$. The Gateaux derivative of the unit loss is $s(W; h,f,\Lambda_0) = h(X)\nabla L\{f(X), W; \Lambda_0\}$.

\subsection{Breslow Estimator}

Define the at-risk indicator $Y(t) = \mathbf{1}\{T \geq t\}$ and counting process $N(t) = \mathbf{1}\{T \leq t,\, D=1\}$, which we treat as random process over time $t$, e.g. $\Env[Y(t)] = n^{-1}\sum_{i=1}^n\mathbf{1}\{t_i \geq t\}$ averages over the observed values of $\{t_i\}_{i=1}^n$.
Let
\begin{align*}
    S_n^{(0)}(t; f) = \Env[Y(t) \exp\{f(X)\}],
\end{align*}
where the letter `S' is taken to mean `sum', not to be confused with the survival function.
The Breslow estimator of the cumulative hazard is
\begin{align*}
    \hat{\Lambda}_{n,f}(t) = \Env\left[\frac{N(t)}{S_n^{(0)}(T; f)}\right].
\end{align*}
Note that this estimator of the baseline hazard depends on the log-hazard ratio $f$.

\subsection{Efficient Influence Function}
\label{app:eif}

Here we derive the efficient influence function for $R'_0(h; f_*) = \E[s(W; h, f_*, \Lambda_0)]$ under the null hypothesis $H_0: f_*=f_0$.
To do so, we write the naive influence function $s(W; h, f_*, \Lambda_0)$ as the stochastic integral
\begin{align*}
    \int \alpha(t, X)~\mathrm{d}M(t)
\end{align*}
where $\alpha(t, X) = -h(X)$ and $\mathrm{d}M(t) = \mathrm{d}N(t) - Y(t)\exp\{f_*(X)\}\mathrm{d}\Lambda_0(t)$ is a martingale.
Applying \citet[Theorem 5.5]{tsiatis_semiparametric_2006}, the efficient influence function is
\begin{align*}
    \varphi(W; h, f_*, P_0) &=\int \{\bar{h}(t) - h(X)\}~\mathrm{d}M(t) \\&= s(W; h, f_*, \Lambda_0) + D\,\bar{h}(T) - \exp\{f_*(X)\}\int_0^T \bar{h}(t)\,\mathrm{d}\Lambda_0(t),
\end{align*}
where
\begin{align*}
    \bar{h}(t) = \frac{\E[h(X)\exp\{f_*(X)\}Y(t)]}{\E[\exp\{f_*(X)\}Y(t)]}.
\end{align*}

\end{document}